\documentclass[lettersize,journal]{IEEEtran}

\usepackage{amsmath,amsfonts}
\usepackage{amssymb}
\usepackage{amsthm}
\usepackage{algorithmic}
\usepackage{algorithm}
\usepackage{array}
\usepackage{textcomp}
\usepackage{stfloats}
\usepackage{url}
\usepackage{verbatim}
\usepackage{graphicx}
\usepackage{subfigure}
\usepackage{cite}
\usepackage{multirow}

\usepackage{setspace}
\usepackage{color}
\usepackage{hyperref}
\usepackage{caption}
\allowdisplaybreaks[4]
\newtheorem{thm}{Theorem}
\newtheorem{lem}{Lemma}

\newtheorem{cor}{Corollary}

\begin{document}

\title{Adversarial Bandits with Multi-User Delayed Feedback: Theory and Application}

\author{Yandi Li,
    Jianxiong Guo,~\IEEEmembership{Member,~IEEE},
    Yupeng Li,
    Tian Wang,
    and Weijia Jia,~\IEEEmembership{Fellow,~IEEE}
	\thanks{This work is an extended version of the paper \cite{li2023Modified}, which has been accepted at the 29th International Computing and Combinatorics Conference (COCOON 2023), Hawaii, USA.
 
    Yandi Li is with the Hong Kong Baptist University, Hong Kong, and also with the Guangdong Key Lab of AI and Multi-Modal Data Processing, Department of Computer Science, BNU-HKBU United International College, Zhuhai 519087, China. (E-mail: liyandi@uic.edu.cn)
	
	Jianxiong Guo, Tian Wang, and Weijia Jia are with the Advanced Institute of Natural Sciences, Beijing Normal University, Zhuhai 519087, China, and also with the Guangdong Key Lab of AI and Multi-Modal Data Processing, BNU-HKBU United International College, Zhuhai 519087, China. (E-mail: jianxiongguo@bnu.edu.cn; cs\_tianwang@163.com; jiawj@bnu.edu.cn)

    Yupeng Li is with the Department of Interactive Media, Hong Kong Baptist University, Hong Kong. (E-mail: ivanypli@gmail.com)
	
	\textit{(Corresponding author: Jianxiong Guo.)}
	}
\thanks{Manuscript received April xxxx; revised August xxxx.}}

\markboth{}%
{Shell \MakeLowercase{\textit{et al.}}: A Sample Article Using IEEEtran.cls for IEEE Journals}


\maketitle

\begin{abstract}
The multi-armed bandit (MAB) models have attracted significant research attention due to their applicability and effectiveness in various real-world scenarios such as resource allocation, online advertising, and dynamic pricing. As an important branch, the adversarial MAB problems with delayed feedback have been proposed and studied by many researchers recently where a conceptual adversary strategically selects the reward distributions associated with each arm to challenge the learning algorithm and the agent experiences a delay between taking an action and receiving the corresponding reward feedback. However, the existing models restrict the feedback to be generated from only one user, which makes models inapplicable to the prevailing scenarios of multiple users (e.g. ad recommendation for a group of users). In this paper, we consider that the delayed feedback results are from multiple users and are unrestricted on internal distribution. In contrast, the feedback delay is arbitrary and unknown to the player in advance. Also, for different users in a round, the delays in feedback have no assumption of latent correlation. Thus, we formulate an adversarial MAB problem with multi-user delayed feedback and design a modified EXP3 algorithm MUD-EXP3, which makes a decision at each round by considering the importance-weighted estimator of the received feedback from different users. On the premise of known terminal round index $T$, the number of users $M$, the number of arms $N$, and upper bound of delay $d_{max}$, we prove a regret of $\mathcal{O}(\sqrt{TM^2\ln{N}(N\mathrm{e}+4d_{max})})$. Furthermore, for the more common case of unknown $T$, an adaptive algorithm AMUD-EXP3 is proposed with a sublinear regret with respect to $T$. Finally, extensive experiments are conducted to indicate the correctness and effectiveness of our algorithms.
\end{abstract}


\section{Introduction}
\IEEEPARstart{T}{he} multi-armed bandit (MAB) problems are a collection of sequential decision-making problems that attract increasing attention for substantial application scenarios such as task offloading in edge computing \cite{zhou2020online}, edge resource scheduling \cite{han2021cache}, worker selection for mobile crowdsourcing \cite{chen2022learning}, recommendation systems \cite{liu2018contextual} and online advertising \cite{avadhanula2021stochastic}. They refer to adopting an action at each round and collecting feedback information for subsequent action selection. 

It is known that stochastic bandits and adversarial bandits are two underpinnings of multi-armed bandits. In conventional stochastic bandits, the feedback generated from actions is assumed to follow a fixed but unknown distribution, where the player can gradually estimate the expected feedback through continuous interaction with the environment. Yet, the potential feedback distribution tends to alter with round, which gives rise to adversarial bandits where the feedback distribution for each arm is chosen adversarially. Moreover, in a more practical situation, the feedback can suffer variable delays before being observed by the player in the real world. Thus, it is not uncommon for people to face adversarial bandit problems with delayed feedback settings \cite{wan2022bounded,masoudian2022best,bistritz2019online,cesa2019delay,thune2019nonstochastic}. However, the existing works only consider single-user feedback situations but not feedback from multiple users at a round. As an example, for online advertising, an advertisement is taken out for multiple users at a round and the delays in receiving those users' feedback are different. The new advertisement has to be put in before receiving all the user feedback of the last-round advertisement. In this situation, the user group can vary with rounds, corresponding to arbitrarily changed feedback in adversarial bandits.

In this paper, we focus on oblivious adversary bandits with multi-user obliviously delayed feedback, where the feedback and delays for all arms, all users, and all rounds are arbitrarily chosen in advance. Specifically, the player executes an arm out of total $N$ arms on $M$ distinct users at round $t$. Then, the feedback from user $j$ is observed at round $t+d_{t}^{j}$, where $d_{t}^{j}$ is the delay of the feedback. This model can be adopted in many real-world applications. For example, as shown in Figure \ref{fig_eg}, think of a scenario where a recommendation agent chooses an item to recommend to a group of users at a round, such as recommending an online advertisement to the audience. Assume there are a total of $N$ items that can be selected, and the recommendation agent selects only one item each round to recommend based on the received feedback from users. In the online advertising case, for instance, the feedback of a user can be his or her comments or rating scores on the recommended items, reflecting the interest the user has in this advertisement. Due to the fluctuation introduced by the randomness in user preference, a group of users can be considered to share a feedback distribution at a round and each single-user feedback is equivalently sampled from it \cite{yuan2014generative,li2023online}. The user group can be time-varying in terms of its constitution and preference of users (i.e. the constituting users of the target group might change with rounds and so does their preference for those recommended items), which results in dynamic feedback distributions even for a fixed item. Not only do different items correspond to different feedback distributions, but even the same item could result in different feedback distributions at different rounds due to the time-varying nature of the user group. Specifically, the feedback distribution with respect to taking item $i$ at round $t_1$ could evolve and be different from the distribution of taking the same item at round $t_2$. Besides, the users who receive the recommendation may not return the feedback instantly but after variably delayed rounds that are particular to the users and can hardly be known by the agent in advance. Thus the agent can only make recommendation decisions based on the already received feedback. Therefore, the time-varying user group along with delayed feedback brings great challenges to developing optimal policy. This case can be in essence formulated as the proposed multi-user delayed-feedback adversarial bandit models. That is, the recommendation agent is seen as the player, while the candidate items act as the arms in our model.

\begin{figure}[!t]
	\centering
    \includegraphics[width=\linewidth]{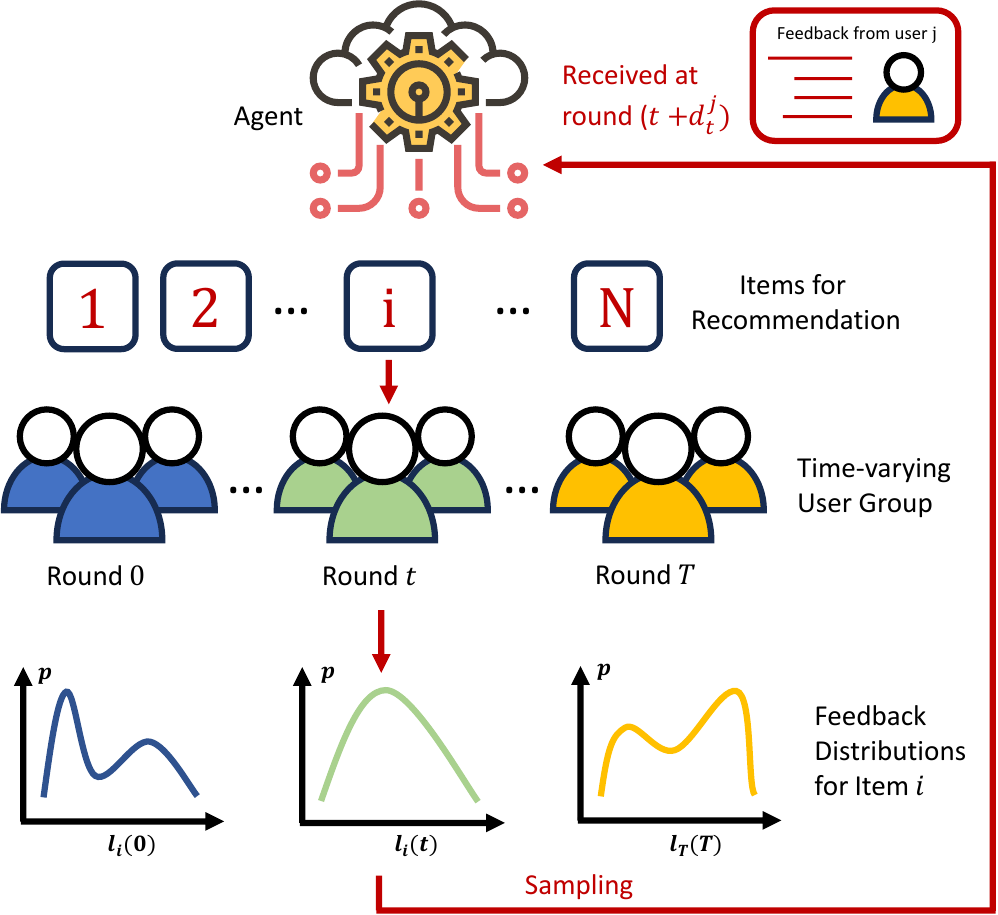}
	\caption{The illustrative example of recommending items to the time-varying user group. At round $t$, the agent takes the item $i$ as the action, and then each user $j$ in user group (green group) gives feedback $l^j_i(t)$ according to the feedback distribution for item $i$ at round $t$. Finally, this feedback is received by the agent at round $t+d_{t}^{j}$ with delay $d_{t}^{j}$.}
	\label{fig_eg}
\end{figure}

To solve this adversarial multi-armed bandit problem with multi-user delayed feedback, we propose a modified EXP3 algorithm \cite{auer1995gambling} named MUD-EXP3 which effectively distinguishes the potential optimal arms as the terminal round index $T$ and the upper bound of delay $d_{max}$ are known in advance, corresponding to those informed situations in reality. At each round, the player chooses an arm according to the importance-weighted estimator of the received feedback from different users. In addition, we conduct a detailed theoretical analysis and prove the upper bound of regret for the proposed algorithm. Moreover, for less-informed cases of unknown $T$, which is pervasive in the real world, MUD-EXP3 cannot guarantee a sublinear regret upper bound according to Theorem \ref{thm1}. For example, the duration of a recommendation project would usually not be fixed at the very beginning but relies on the budget and the subsequent commercial effects, hence the terminal index $T$ is undetermined. To overcome this, we adapt MUD-EXP3 by using a doubling trick method \cite{cesa1997use,bistritz2019online} with a dynamic learning rate to achieve a sublinear regret, which removes the dependence upon $T$. Finally, we design and conduct simulated experiments with diverse environment settings to investigate the performance of the algorithms under different cases.

In the case above, the variation in the feedback distribution can resort to the problem of adversarial bandit \cite{auer1995gambling}, a.k.a. the non-stochastic bandits, of which the feedback of a certain action can arbitrarily change over rounds as if they are selected by an adversary. Furthermore, it experiences delays between conducting actions and receiving feedback. This triggers off the problem of adversarial bandits with delayed feedback \cite{wan2022bounded,masoudian2022best,bistritz2019online,cesa2019delay,thune2019nonstochastic}. However, the existing works only consider single-user feedback situations but not feedback from multiple users at a round. As an example, for online advertising, an advertisement is taken out for multiple users at a round and the delays in receiving those users' feedback are different. The new advertisement has to be put in before receiving all the user feedback of the last-round advertisement. In this situation, the user group can vary with rounds, corresponding to arbitrarily changed feedback in adversarial bandits. In this paper, we focus on oblivious adversary bandits with multi-user obliviously delayed feedback, where the feedback and delays for all arms, all users, and all rounds are arbitrarily chosen in advance. Specifically, the player executes an arm out of total $N$ arms on $M$ distinct users at round $t$. Then, the feedback from user $j$ is observed at round $t+d_{t}^{j}$, where $d_{t}^{j}$ is the delay of the feedback. To solve this problem, we propose a modified EXP3 algorithm \cite{auer1995gambling} named MUD-EXP3 which effectively distinguishes the potential optimal arms as the terminal round index $T$ and the upper bound of delay $d_{max}$ are known in advance, corresponding to those informed situations in reality. At each round, the player chooses an arm according to the importance-weighted estimator of the received feedback from different users. In addition, we conduct a detailed theoretical analysis and prove the upper bound of regret for the proposed algorithm. Moreover, for less-informed cases of unknown $T$, which is pervasive in the real world, MUD-EXP3 cannot guarantee a sublinear regret upper bound according to Theorem \ref{thm1}. For example, the duration of a recommendation project would usually not be fixed at the very beginning but relies on the budget and the subsequent commercial effects, hence the terminal index $T$ is undetermined. To overcome this, we adapt MUD-EXP3 by using a doubling trick method \cite{cesa1997use,bistritz2019online} with a dynamic learning rate to achieve a sublinear regret, which removes the dependence upon $T$. Finally, we design and conduct simulated experiments with diverse environment settings to investigate the performance of the algorithms under different cases.

This manuscript is a journal extension to our previous conference article \cite{li2023Modified}. Compared with our conference version, in the present paper, we further extend our theory in order to cover a more practical situation, and we evaluate the ability of our methods by conducting numerical simulations. First of all, realizing the limitation that the algorithm proposed in our conference paper relies on the prior information of the round horizon, we propose an adaptive algorithm (i.e. Algorithm \ref{AMUD-EPX3}) accommodated to the less-informed cases in the present paper. Also, a detailed analysis of the regret is enunciated (Subsection \ref{Reg of AMUD-EXP3}). In addition, extensive numerical evaluations of the two proposed algorithms are demonstrated along with the state-of-the-art baselines (Section \ref{sec6}). The main contributions of the paper are summarized as the following points:
\begin{itemize}
    \item We introduce a sequential decision-making problem with multi-user delayed feedback where no assumptions are made for the distributions of feedback and delays, which is ubiquitous in real-life situations. Then, we model it by using an adversarial bandit framework considering the trait of varying individual loss.
    \item We propose a modified EXP3 algorithm to adapt to our problem setting, which adopts an importance-weighted estimating method for the received feedback from different users in order to effectively balance exploration and exploitation.
    \item For the case where the terminal round index $T$ is unknown as a prior, we propose an adaptive algorithm to address this problem based on a doubling trick considering delays.
    \item Sound and detailed theoretical analysis are presented to derive the regret upper bound of the proposed algorithms, achieving sublinear properties with regard to the terminal round index $T$.
    \item We conduct extensive simulations and demonstrate the superior performance of the proposed algorithms over several state-of-the-art non-stochastic bandit baselines. 
\end{itemize}

The remainder of this paper is organized as follows. Section \ref{sec2} reviews the previous related works from the classic multi-armed bandit algorithms to recent ones with delayed feedback settings. Section \ref{sec3} formulates the problem we propose into an adversarial multi-armed bandit problem and Section \ref{sec4} introduces two algorithms we propose to solve this problem in different cases. Section \ref{sec5} provides a detailed analysis of the regret upper bound for our proposed algorithms. Section \ref{sec6} presents and discusses the numerical evaluation results and Section \ref{sec7} concludes.

\section{Related Work}\label{sec2}
In this section, we first briefly review the development of general multi-armed bandit approaches in terms of stochastic bandits, adversarial bandits, and contextual bandits. Then, we focus on surveying the bandit problems with delayed feedback settings which are closely related to our work in this article. Finally, a more complex setting derived from the delayed feedback, named ``composite anonymous feedback'', is covered in short.

\subsection{Multi-Armed Bandit}
Originating from the field of probability theory, Multi-Armed Bandit algorithms present a class of reinforcement learning techniques designed to address the exploration-exploitation dilemma. These algorithms have found applications in various domains, including online advertising \cite{auer2002finite} and clinical trials \cite{kaufmann2013information}. The classical MAB problem assumes stochastic rewards, where each arm provides rewards drawn from an unknown probability distribution. Seminal works like \cite{robbins1952some} laid the groundwork for understanding exploration-exploitation trade-offs in sequential decision-making. 

In the stochastic bandit setting, the reward distributions associated with arms remain stationary over rounds. Pioneering algorithms such as UCB1 \cite{auer2002finite} and Thompson Sampling \cite{agrawal2012analysis} have demonstrated near-optimal strategies for balancing exploration and exploitation in stochastic environments. Based on UCB1, Audibert \textit{et al.} \cite{audibert2009minimax} proposed MOSS, also known as the Minimax Optimal Strategy in the Stochastic scenario, which incorporated a minimax approach to ensure that it performs well in the worst-case scenario among a class of stochastic bandit problems. Maillard \textit{et al.} \cite{maillard2011finite} then proposed the KL-UCB algorithm that incorporated the Kullback-Leibler divergence, a measure of information loss in probability distributions. It provided a more nuanced exploration-exploitation balance, allowing it to better handle uncertainties in the reward probabilities of different arms and leading to improved decision-making in scenarios where accurate estimation of these probabilities was crucial. Alongside KL-UCB, Kaufmann \textit{et al.} \cite{kaufmann2012bayesian} created a Bayesian version of UCB that was explicitly designed, proving to be an advanced algorithm with asymptotic efficiency. Empirical results demonstrated its superior performance over KL-UCB.

In the adversarial bandit setting, the rewards for each arm are chosen adversarially, posing a significant challenge. Algorithms like Exp3 \cite{auer1995gambling} have been developed to handle the adversarial nature of the problem, showcasing the adaptability of MAB algorithms. Auer \textit{et al.} \cite{auer2002nonstochastic} developed an enhanced version of Exp3, named Exp4 (Exp3 with expert advice), incorporating specific advice, represented by a selection coefficient, for each arm." Then Zoghi \textit{et al.} \cite{zoghi2014relative} proposed the RUCB algorithm (Relative Upper Confidence Bound) extending UCB by incorporating pairwise probabilities to select promising arms, achieving a sharp finite-time regret bound on a broad class of dueling bandit problems.

Moreover, contextual bandit extends the MAB framework by incorporating contextual information. Each arm is associated with a context, and the reward depends not only on the arm but also on the context. Algorithms such as LinUCB \cite{li2010contextual}, NeuralBandit \cite{allesiardo2014neural}, Thompson Sampling for Contextual Bandits \cite{agrawal2013thompson} and KernelUCB \cite{valko2013finite} address the contextual nature of the problem, enabling personalized decision-making based on available context information. 

Overall, researchers have extensively investigated the theoretical foundations and practical implementations of Multi-Armed Bandit algorithms, showcasing their effectiveness in adaptive decision-making scenarios. 

\subsection{Delayed Feedback}\label{sec2-2}
For stochastic MAB, Dudik \textit{et al.} \cite{dudik2011efficient} first introduced the delay mechanism in the contextual bandit setting and proposed an algorithm achieving $\mathcal{O}(\sqrt{K\log(NT)}(d+\sqrt{T}))$ regret bound. Joulani \textit{et al.} \cite{joulani2013online} investigated the impact of delayed feedback in online learning, particularly in scenarios like web advertisement and distributed learning, where feedback arrives with delays. The study revealed that delays lead to increased regret in a multiplicative manner for adversarial problems and additively for stochastic problems. The paper presented meta-algorithms that adapt existing non-delayed algorithms to handle delays efficiently, and also introduced modified versions of the UCB algorithm specifically designed for stochastic bandit problems with delayed feedback, offering lower complexity than the general meta-algorithms. Joulani \textit{et al.} \cite{joulani2016delay} studied delayed feedback under full information setting rather than adversarial bandits, and proved regret bound of $\mathcal{O}(\sqrt{(T+D)\ln{N}})$ by reducing the problem from non-delayed feedback full information setting, where $D=\sum_{t=1}^{T}d_t$ is delay sum. Vernade \textit{et al.} \cite{vernade2017stochastic} proposed the delayed stochastic bandit model based on the framework by Chapelle \cite{chapelle2014modeling} with a partially observed feedback setting, assuming known delay distribution and bounded expected delay. Gael \textit{et al.} \cite{gael2020stochastic} weakened the strong assumptions on the delay distributions with only a bound on the tail of the delay, and designed a UCB-based algorithm to solve it.

For non-stochastic MAB, there also exist studies considering delayed feedback. Cesa-Bianchi \textit{et al.} \cite{cesa2019delay} proposed a cooperative version of the EXP3 for delayed feedback under the bandit setting. They achieved regret upper bound of $\mathcal{O}(\sqrt{(NT+D)\ln{N}})$. Based on the study by Cesa-Bianchi \textit{et al.} \cite{cesa2019delay}, Thune \textit{et al.} \cite{thune2019nonstochastic} proposed a wrapper algorithm to eliminate the restriction on $T$ and $D$ for adversarial bandit setting with delayed feedback, allowing regret bound of $\mathcal{O}(\sqrt{(NT+D)\ln{N}})$ for unknown $T$ and $D$. Bistritz \textit{et al.} \cite{bistritz2019online} also proposed a modified EXP3 algorithm for the same problem as Thune \textit{et al.} \cite{thune2019nonstochastic} and proved the Nash Equilibrium for a two-player zero-sum game with this algorithm. Zimmert and Seldin \cite{zimmert2020optimal} presented a Follow the Regularized Leader algorithm for adversarial bandits with arbitrary delays and achieved the upper bound of $\mathcal{O}(\sqrt{NT}+\sqrt{DT\log N})$ on the regret, requiring no prior knowledge of $D$ or $T$. In recent years there has been an increasing interest in algorithms that perform well in both regimes with no prior knowledge of stochastic or adversarial regimes \cite{bubeck2012best,wei2018more}, known as best-of-both-worlds problems. Masoudian \textit{et al.} \cite{masoudian2022best} followed Zimmert and Seldin \cite{zimmert2020optimal} and introduce delayed feedback setting to the best-of-both-worlds problem. They proposed a slightly modified algorithm and achieved a best-of-both-worlds regret guarantee for both adversarial and stochastic bandits.

\subsection{Composite Anonymous Feedback}
On the basis of delayed feedback setting, Cesa-Bianchi \textit{et al.} \cite{cesa2018nonstochastic} raised a more complex problem, where the feedback is both anonymous and delayed, and can be partly observed at different rounds. It means the player can only observe the sum of partial feedback generated from different rounds. They proposed a general reduction technique that enables the conversion of a standard bandit algorithm to operate in this harder setting. Pike-Burke \textit{et al.} \cite{pike2018bandits} also investigated the setting in which the learner observed the
sum of rewards that arrive at the same round, but they assumed that the expected delay is known. Wang \textit{et al.} \cite{wang2021adaptive} proposed a modified EXP3 algorithm for the problem in \cite{cesa2018nonstochastic} with non-oblivious delays, requiring no knowledge of delays in advance and achieving an $\mathcal{O}(\sqrt{d+N\log{N}}T^{\frac{2}{3}})$ regret. Wan \textit{et al.} \cite{wan2022bounded} extended the problem into anonymous delayed feedback with non-oblivious loss and delays and proved both the lower and upper bounds.

\subsection{Discussion}
Our work is distinct from the existing works. The previous works mainly consider the bandit problem where the feedback arrives instantly. A few of them incorporate delayed feedback in an adversarial bandit environment, such as the representative works reviewed in Section \ref{sec2-2}, which are more challenging than stochastic bandit environments. However, all of them study the single-user setting, while the real applications are commonly associated with multiple users which would certainly complicate the problem and bring more challenges, let alone many of them assume distributions of delays. Our work solves the multi-user adversarial bandit problem with delayed feedback, where each user is likely to generate delayed feedback that is independent of others with no assumption of the delay distribution. Therefore, we largely approximate real-world applications by addressing this new challenging problem. With the application-oriented enhancements of this problem, our work can be further employed in real-world scenarios such as complex network management in mobile communication \cite{zhang2022optimal} and task offloading in intelligent edge computing \cite{ma2023reliability} with dynamic environments.

\section{Problem Formulation}\label{sec3}
In this section, we formulate the problem as an adversarial bandit. Suppose the time is discretized into consecutive rounds. We consider an adversarial multi-armed bandit environment with $N$ arms where the player selects an arm $A_t$ at round $t$ and the corresponding feedback is generated from $M$ individual users. We use the notation $[K]=\{1,2,\cdots,K\}$ for brevity, then we define the set of arm indexes as $\mathcal{N}=\{i\vert i\in [N]\}$, the set of user indexes as $\mathcal{M}=\{j\vert j\in [M]\}$ and the set of round indexes as $\mathcal{T}=\{t\vert t\in [T]\}$. We use the \textit{loss} rather than the \textit{reward} to represent the feedback, denoted as $l_{i}^{j}(t)$ for loss from user $j$ by selecting arm $i$ at round $t$. The loss $l_{A_t}^{j}(t)$ is observed by the player after $d_t^j$ rounds where the delay $d_t^j$ is a positive integer (i.e., $d_t^j\geq 1$) and $d_{max}=\max\{d_t^j\}$. In other words, the player will observe a bunch of feedback losses $\{l_{A_s}^{j}(s)\vert s+d_s^j=t\}$ at round $t$. Without loss of generality, we assume $l_{i}^{j}(t)\in [0,1]$. Note that there is no restriction on the distribution of $d_t^j$ for generality. The losses of arms and the delays are arbitrarily chosen by an adversary prior to the start of the game, which is known as an oblivious adversary. 

The objective is to find a policy of the player for the sequential arm selection in order to approximate the performance of the best fixed arm in hindsight. We use the expected \textit{regret} as the measure of the approximation, which is defined as the difference between the expected cumulative loss induced by the player and the cumulative loss of the best arm in hindsight, as shown below.
\begin{equation}\label{ori_regret}
    \mathcal{R} = \mathbb{E}\left[\sum_{t=1}^{T}\sum_{j=1}^{M}l_{A_t}^{j}(t)\right] - \min_{i\in\mathcal{N}}\sum_{t=1}^{T}\sum_{j=1}^{M}l_{i}^{j}(t).
\end{equation}

\begin{algorithm}[!t]
    \caption{\text{MUD-EXP3}}
    \begin{algorithmic}[1]\label{MUD-EPX3}
        \renewcommand{\algorithmicrequire}{\textbf{Input:}}
    	\renewcommand{\algorithmicensure}{\textbf{Output:}}
        \REQUIRE $\mathcal{N}$, $\mathcal{M}$, $\mathcal{T}$, learning rate $\eta$, upper bound on the delays $d_{max}$;
        \STATE Truncate the learning rate: $\eta'=\min_{i\in\mathcal{N}}\left\{\eta, \frac{1}{MN\mathrm{e}(\Delta+1)}\right\}$;
        \STATE Initialize $\hat{L}_{i}(1)=0$, $p_i(1)=\frac{1}{N}$ for any arm $i\in\mathcal{N}$;
        \FOR{$t=1,2,\cdots,T$}
            \STATE Draw an arm $A_t\in\mathcal{N}$ according to the distribution $\boldsymbol{p}(t)$;
            \STATE Observe a set of delayed losses $\{l_{A_s}^{j}(s)\vert (s,j)\in \Phi_t\}$;
            \STATE Update the cumulative estimated loss $\hat{L}_{i}(t)$ for all $i\in \mathcal{N}$:
            \STATE \quad $\hat{l}_{i}^{j}(s) = \frac{\mathbb{I}\{A_s=i\}\cdot l_{i}^{j}(s)}{p_i(s)}$, $(s,j)\in \Phi_t$; 
            \STATE \quad $\ell_{i}(t)=\sum_{(s,j)\in \Phi_t}\hat{l}_{i}^{j}(s)$;
            \STATE \quad $\hat{L}_{i}(t)=\hat{L}_{i}(t-1)+\ell_{i}^{j}(t)$
            \STATE Update the distribution $\boldsymbol{p}(t+1)$:
            \STATE \quad $W_{i}(t)=\exp (-\eta^{\prime}\hat{L}_{i}(t))$; 
            \STATE \quad $p_{i}(t+1)=\frac{W_{i}(t)}{\sum_{k=1}^{N}W_{k}(t)}$ for all $i\in \mathcal{N}$;
            \STATE \quad$\boldsymbol{p}(t+1)=[p_1(t+1),\cdots,p_N(t+1)]$;
        \ENDFOR
    \end{algorithmic}
\end{algorithm}

\section{Algorithm}\label{sec4}
In this section, two algorithms are proposed for our problem. Depending on whether the parameter of $T$ is known or not, we tactically choose one of them to utilize.

\subsection{Multi-User Delayed EXP3}
The first algorithm we propose is named MUD-EXP3 (\textbf{M}ulti-\textbf{U}ser \textbf{D}elayed \textbf{EXP3}). MUD-EXP3 is devised based on the well-known EPX3 algorithm and takes into account the multi-user delayed feedback. The detailed algorithm is laid out in Algorithm \ref{MUD-EPX3}. We assume the terminal round $T$ and the upper bound on the delays $d_{max}$ are known. The input learning rate $\eta$ has to be truncated first if it is over $1/(MN\mathrm{e}(\Delta+1))$ for the guarantee of upper bound on regret, where $\Delta\geq d_{max}$. 

At each round, MUD-EXP3 chooses an arm according to the softmax distribution $\boldsymbol{p}(t)$ derived from the cumulative estimated rewards, i.e. the minus cumulative expected losses, of each arm, where $\boldsymbol{p}(t)=[p_1(t),\cdots,p_N(t)]$ and the cumulative estimated loss of arm $i$ is denoted by $\hat{L}_{i}(t)$. Note that the rounds and the users that contribute to the received losses at a round are arbitrary. For convenience, we first introduce the set of round-user pairs whose feedback losses are observed at round $t$ as $\Phi_t$, and represent the set of round-user pairs whose feedback losses are received out of the terminal round $T$ as $\Omega$, where $\Phi_t=\{(s,j)\vert s+d_s^j=t\}$ and $\Omega=\{(s,j)\in \Phi_t (t>T)\}$. The importance-weighted estimator $\hat{l}_{i}^{j}(t)$ is adopted to estimate the loss $l_{i}^{j}(t)$ for introducing the exploration into the algorithm, which is defined as
\begin{equation}
    \hat{l}_{i}^{j}(t) = \frac{\mathbb{I}\{A_t=i\}\cdot l_{i}^{j}(t)}{p_i(t)}.\notag
\end{equation}
We wrap up the sum of the estimated losses at round $t$ as $\ell_{i}(t)=\sum_{(s,j)\in \Phi_t}\hat{l}_{i}^{j}(s)$. Then, the cumulative estimated loss is
\begin{equation}
    \hat{L}_{i}(t)=\hat{L}_{i}(t-1)+\ell_{i}(t).\notag
\end{equation}
Let $W_{i}(t)=\exp(-\eta^{\prime}\hat{L}_{i}(t))$. The probability of choosing arm $i$ at round $t$ which is conditioned on the history observed after $t-1$ rounds is of the softmax function
\begin{equation}
    p_{i}(t)=\frac{W_{i}(t-1)}{\sum_{k=1}^{N}W_{k}(t-1)}.\notag
\end{equation}
Define the filtration $\mathcal{F}_t$ as $\mathcal{F}_t=\sigma\left(\left\{A_s\vert s+d_s^j\leq t\right\}\right)$, where $\sigma\left(\left\{A_s\vert s+d_s^j\leq t\right\}\right)$ denotes the $\sigma$-algebra generated by the random variables in $\left\{A_s\vert s+d_s^j\leq t\right\}$. Note that the distribution $\boldsymbol{p}(t)$ is $\mathcal{F}_{t-1}$-measurable since $\boldsymbol{p}(t)$ is a function of all feedback received up to round $t-1$.

\subsection{Adaptive Multi-User Delayed EXP3}
In real-world scenarios, it is normally hard to acquire the information of terminal round and delays in advance. According to the Theorem \ref{thm1}, if with unknown $T$, we cannot assign $\eta$ in advance to guarantee the regret upper bound achieved through our analysis. However, in many real-world situations, the horizon index $T$ can barely be obtained beforehand. Therefore, a novel doubling trick is adopted to solve this problem. We define the number of missing feedback samples at round $t$ as $V_t$. If one feedback sample $l_{i}^{j}(s)$ has not been observed at round $t$ (i.e. $s+d_s^j>t$), it is considered to contribute one towards $V_t$. We introduce the concept of epoch to divide the timeline, defining the $\varepsilon$-th epoch $\mathcal{T}_\varepsilon$ as the set of rounds that satisfying $\mathcal{T}_\varepsilon=\left\{t\vert 2^{\varepsilon-1}M\leq\sum_{\tau=1}^{t}V_{\tau}<2^{\varepsilon}M\right\}$. Furthermore, we define $T_{\varepsilon}=\max\mathcal{T}_\varepsilon$ and $T_0=0$ for analysis convenience.

We present an adaptive MUD-EXP3 algorithm called AMUD-EXP3 in Algorithm \ref{AMUD-EPX3} to address the problem with no prior knowledge of $T$. Compared with MUD-EXP3, AMUD-EXP3 updates $V_t$ each round and increase the epoch index $\varepsilon$ by one if $\sum_{\tau=1}^{t}V_{\tau}\geq 2^{\varepsilon}\cdot M$. In addition, we use the variable learning rate $\eta_{\varepsilon}=\frac{1}{M}\sqrt{\frac{\ln{N}}{2^{\varepsilon}}}$ adapting to epochs. The notation is aligned with Algorithm \ref{MUD-EPX3} besides the subscripts associated with epoch indexes (e.g., $\Omega_{\varepsilon}$ denotes $\Omega$ for epoch $\varepsilon$).

\begin{algorithm}[!t]
    \caption{\text{AMUD-EXP3}}
    \begin{algorithmic}[1]\label{AMUD-EPX3}
        \renewcommand{\algorithmicrequire}{\textbf{Input:}}
    	\renewcommand{\algorithmicensure}{\textbf{Output:}}
        \REQUIRE $\mathcal{N}$, $\mathcal{M}$;
        \STATE Initialize $\hat{L}_{i}(1)=0$ and $p_i(1)=\frac{1}{N}$ for any arm $i\in\mathcal{N}$, epoch index $\varepsilon=0$, learning rate $\eta=1$;
        \FOR{$t=1,2,\cdots,T$}
            \STATE Draw an arm $A_t\in\mathcal{N}$ according to the distribution $\boldsymbol{p}(t)$;
            \STATE Observe a set of delayed losses $\{l_{A_s}^{j}(s)\vert (s,j)\in \Phi_t\}$;
            \STATE Update the number of missing samples $V_t=M\cdot t-\sum_{\tau=1}^{t}\lvert\Phi_{\tau}\rvert$;
            \STATE If $\sum_{\tau=1}^{t}V_{\tau}\geq 2^{\varepsilon}\cdot M$, then $\varepsilon=\varepsilon+1$ and $\hat{L}_{i}(t)=0$ for any arm $i\in\mathcal{N}$;
            \STATE Assign $\eta=\frac{1}{M}\sqrt{\frac{\ln{N}}{2^{\varepsilon}}}$;
            \STATE Update the cumulative estimated loss $\hat{L}_{i}(t)$ for all $i\in \mathcal{N}$:
            \STATE \quad $\hat{l}_{i}^{j}(s) = \frac{\mathbb{I}\{A_s=i\}\cdot l_{i}^{j}(s)}{p_i(s)}$, $(s,j)\in \Phi_t$; 
            \STATE \quad $\ell_{i}(t)=\sum_{(s,j)\in \Phi_t}\hat{l}_{i}^{j}(s)$;
            \STATE \quad $\hat{L}_{i}(t)=\hat{L}_{i}(t-1)+\ell_{i}^{j}(t)$;
            \STATE Update the distribution $\boldsymbol{p}(t+1)$:
            \STATE \quad $W_{i}(t)=\exp (-\eta^{\prime}\hat{L}_{i}(t))$; 
            \STATE \quad $p_{i}(t+1)=\frac{W_{i}(t)}{\sum_{k=1}^{N}W_{k}(t)}$ for all $i\in \mathcal{N}$;
            \STATE \quad$\boldsymbol{p}(t+1)=[p_1(t+1),\cdots,p_N(t+1)]$;
        \ENDFOR
    \end{algorithmic}
\end{algorithm}

\section{Regret Analysis}\label{sec5}
In this section, we establish the regret upper bound for MUD-EXP3 and AMUD-EXP3. The regret is defined as Eq. (\ref{ori_regret}) in terms of the optimal arm in hindsight. However, we can hardly acquire the optimal arm in practice. Thus, we transform the regret $\mathcal{R}$ into $\mathcal{R}_i$ as follows.
\begin{equation}
    \mathcal{R}_i = \mathbb{E}\left[\sum_{t=1}^{T}\sum_{j=1}^{M}l_{A_t}^{j}(t)\right] - \sum_{t=1}^{T}\sum_{j=1}^{M}l_{i}^{j}(t), \notag
\end{equation}
which considers the difference between the proposed policy and any fixed arm $i$. Note that the bounding $\mathcal{R}_i$ is sufficient for bounding $\mathcal{R}$ since a regret bound that is applicable for any fixed arm $i$ definitely fits the optimal arm. In the present, our target changes to bound $\mathcal{R}_i$ and then to transfer the bound to $\mathcal{R}$. The analysis in this section is conducted by partially following some techniques of the existing works \cite{lattimore2020bandit,cesa2019delay,thune2019nonstochastic,bistritz2019online}.

\subsection{Regret of MUD-EXP3}
We will first prove our regret upper bound of MUD-EXP3 and will first introduce some auxiliary lemmas and intermediate theorems before reaching the analysis of the eventual regret bound. For a brief explanation, Lemma \ref{lem1} and Lemma \ref{lem2} serve Lemma \ref{lem3}, Corollary \ref{cor1} serves Lemma \ref{lem4}, and Lemma \ref{lem3} and \ref{lem4} support Theorem \ref{thm1}. 

\begin{lem}\label{lem1}
    Under the setting of MUD-EXP3, for any round $t\geq1$ and for any arm $i\in \mathcal{N}$, we have the following: 
    \begin{equation*}
        -\eta'p_{i}(t)\cdot\ell_{i}(t) \leq p_{i}(t+1)-p_{i}(t)\leq \eta' p_{i}(t+1)\sum_{k=1}^{N}p_{k}(t)\cdot\ell_{k}(t).
    \end{equation*}
\end{lem}
\begin{proof}
    We start the proof with upper bounding $p_{i}(t+1)-p_{i}(t)$ and then lower bounding it by using the definition of $p_{i}(t)$ in terms of $W_{i}(t)$. According to the definition of $p_{i}(t)$, we can make the following transformation for the upper bound of $p_{i}(t+1)-p_{i}(t)$.
    \begin{align}
        &\quad p_{i}(t+1)-p_{i}(t)= p_{i}(t+1) - \frac{W_{i}(t-1)}{\sum_{k=1}^{N}W_{k}(t-1)} \notag \\
        &= p_{i}(t+1) \notag\\
        &\qquad - p_{i}(t+1)\cdot\frac{\sum_{k=1}^{N}W_{k}(t)}{W_{i}(t-1) \mathrm{e}^{-\eta'\ell_{i}(t)}}\cdot\frac{W_{i}(t-1)}{\sum_{k=1}^{N}W_{k}(t-1)} \notag\\
        &\mathop{\leq}\limits_{(a)} p_{i}(t+1) - p_{i}(t+1)\cdot\frac{\sum_{k=1}^{N}W_{k}(t)}{\sum_{k=1}^{N}W_{k}(t-1)} \notag\\
        &= p_{i}(t+1) - p_{i}(t+1)\cdot\frac{\sum_{k=1}^{N}W_{k}(t-1)\mathrm{e}^{-\eta'\ell_{k}(t)}}{\sum_{k=1}^{N}W_{k}(t-1)} \notag\\
        &= p_{i}(t+1) - p_{i}(t+1)\cdot\sum_{l=1}^{N}\frac{W_{l}(t-1)}{\sum_{k=1}^{N}W_{k}(t-1)}\mathrm{e}^{-\eta'\ell_{l}(t)} \notag\\
        &= p_{i}(t+1) - p_{i}(t+1)\cdot\sum_{l=1}^{N}p_{l}(t)\mathrm{e}^{-\eta'\ell_{l}(t)} \notag\\
        &= p_{i}(t+1)\left(\sum_{l=1}^{N}p_{l}(t) - \sum_{l=1}^{N}p_{l}(t)\mathrm{e}^{-\eta'\ell_{l}(t)}\right) \notag\\
        &= p_{i}(t+1)\sum_{l=1}^{N}p_{l}(t)\left(1 - \mathrm{e}^{-\eta'\ell_{l}(t)}\right) \notag\\
        &\mathop{\leq}\limits_{(b)} p_{i}(t+1)\sum_{l=1}^{N}\eta'p_{l}(t)\cdot\ell_{l}(t). \notag
    \end{align}
    The Ineq. $(a)$ results from $W_{i}(t-1)\exp\left(-\eta'\ell_{i}(t)\right) \leq W_{i}(t-1)$ since for $x\leq 0$ there exists $\exp\left(x\right) \leq 1$ and Ineq. $(b)$ is obtained by using $1-\exp\left(-x\right)\leq x$ for $x\in \mathbb{R}$ with $x=\eta'\ell_{l}(t)$ here. 
    
    We then infer the lower bound
    \begin{align}
        &\quad p_{i}(t+1)-p_{i}(t) = \frac{W_{i}(t-1)\mathrm{e}^{-\eta'\ell_{i}(t)}}{\sum_{k=1}^{N}W_{k}(t)} - p_{i}(t) \notag\\
        &= \frac{W_{i}(t-1)}{\sum_{k=1}^{N}W_{k}(t-1)}\cdot\mathrm{e}^{-\eta'\ell_{i}(t)}\cdot\frac{\sum_{k=1}^{N}W_{k}(t-1)}{\sum_{k=1}^{N}W_{k}(t)} - p_{i}(t) \notag\\
        &= p_{i}(t)\mathrm{e}^{-\eta'\ell_{i}(t)}\cdot\sum_{l=1}^{N}\frac{W_{l}(t-1)}{\sum_{k=1}^{N}W_{k}(t)} - p_{i}(t) \notag\\
        &= p_{i}(t)\mathrm{e}^{-\eta'\ell_{i}(t)}\cdot\sum_{l=1}^{N}\frac{W_{l}(t-1)\mathrm{e}^{-\eta'\ell_{l}(t)}}{\sum_{k=1}^{N}W_{k}(t)}\frac{1}{\mathrm{e}^{-\eta'\ell_{l}(t)}} - p_{i}(t) \notag\\
        &= p_{i}(t)\mathrm{e}^{-\eta'\ell_{i}(t)}\cdot\sum_{l=1}^{N}p_{l}(t+1)\frac{1}{\mathrm{e}^{-\eta'\ell_{l}(t)}} - p_{i}(t) \notag\\
        &\geq p_{i}(t)\mathrm{e}^{-\eta'\ell_{i}(t)}\cdot\sum_{l=1}^{N}p_{l}(t+1) - p_{i}(t) \notag\\
        &\mathop{=}\limits_{(c)} p_{i}(t)\mathrm{e}^{-\eta'\ell_{i}(t)} - p_{i}(t)\notag\\
        &= p_{i}(t)\left(\mathrm{e}^{-\eta'\ell_{i}(t)} - 1\right) \notag\\
        &\mathop{\geq}\limits_{(d)} -\eta'p_{i}(t)\cdot\ell_{i}(t),\notag
    \end{align}
    where Eq. $(c)$ results from $\sum_{l=1}^{N}p_{l}(t+1)=1$ and Ineq. $(d)$ results from the fact of $\exp(x)\geq 1+x$ with $x=-\eta'\ell_{l}(t)$ here.
\end{proof}

Lemma \ref{lem1} derives both the upper bound and lower bound of $p_{i}(t+1)-p_{i}(t)$, which is applied to the proof of Lemma \ref{lem2} so that we can obtain the upper bound of $\frac{p_{i}(t+1)}{p_{i}(t)}$. Before moving to the next lemma, we need to infer a corollary from Lemma \ref{lem1}, upper bounding the sum of the absolute value of $p_{i}(t+1)-p_{i}(t)$ with respect to $p_{i}(t)$ and $\ell_{i}(t)$ for subsequent utilization.

\begin{cor}\label{cor1}
    From Lemma \ref{lem1}, we can derive the following bound: $\sum_{i=1}^{N}\lvert p_{i}(t+1)-p_{i}(t)\rvert \leq 2\eta'\sum_{k=1}^{N}p_{k}(t)\cdot\ell_{k}(t)$.
\end{cor}
\begin{proof}
Due to the space constraint, the proof of Corollary \ref{cor1} is relegated to Appendix \ref{cor1_proof}.
\end{proof}

\begin{lem}\label{lem2}
    Under the setting of MUD-EXP3, if $t\leq T$, for any round $t\geq1$ and for any arm $i\in \mathcal{N}$, if $\eta'\leq \frac{1}{MN\mathrm{e}(\Delta+1)}$ where $\Delta\geq d_{max}$, then we have $p_{i}(t+1) \leq \left(1+\frac{1}{\Delta}\right)p_{i}(t)$.
\end{lem}
\begin{proof}
    We will use the strong mathematical induction method with the assistance of Lemma \ref{lem1} to prove this result. First, we deal with the base case.

    According to the initialization step of MUD-EXP3, $p_i(1)=1/N$. When $t=2$, the maximal increase of $p_i(2)$ compared with $p_i(1)$ should come if a different arm other than $i$ has been chosen at round $t=1$ and all feedback losses $l_{i}^{j}(1)$ are observed with no delays and are of value $1$ for $j\in \mathcal{M}$ at round $t=2$. Hence, we have the following
    \begin{align}
        &\quad p_{i}(2)\leq \frac{1}{N-1+\mathrm{e}^{-\eta'N}}\mathop{\leq}\limits_{(a)} \frac{1}{N-\eta'N}\notag\\
        &= \frac{1}{N}\left(1-\frac{1-\eta'}{1-\eta'}+\frac{1}{1-\eta'}\right) = \frac{1}{N}\left(1+\frac{\eta'}{1-\eta'}\right) \notag\\
        &= p_{i}(1)\left(1+\frac{1}{{1}/{\eta'}-1}\right) \mathop{\leq}\limits_{(b)} p_{i}(1)\left(1+\frac{1}{\Delta}\right),\notag
    \end{align}
    where Ineq. $(a)$ follows since $\exp(x)\geq 1+x$ for $x\in \mathbb{R}$ and Ineq. $(b)$ holds for $\eta'\leq \frac{1}{MN\mathrm{e}(\Delta)}$.

    Next, we assume $p_{i}(t) \leq \left(1+\frac{1}{\Delta}\right)p_{i}(t-1)$ holds for $p_i(2),\cdots,p_i(t)$. Before we prove the case for $p_i(t+1)$, an intermediate result need to be introduced
    \begin{align}
        &\quad \sum_{i=1}^{N}p_{i}(t)\cdot\hat{l}_{i}^{j}(s) =\sum_{i=1}^{N}p_{i}(t)\frac{\mathbb{I}\{A_s=i\}\cdot l_{i}^{j}(S)}{p_i(s)}\notag\\
        &\mathop{\leq}\limits_{(c)} \sum_{i=1}^{N}\frac{p_{i}(t)}{p_{i}(s)}= \sum_{i=1}^{N}\prod_{t=s}^{s+d_s^j}\frac{p_{i}(t)}{p_{i}(t-1)}\notag\\
        &\leq \sum_{i=1}^{N}\left(1+\frac{1}{\Delta}\right)^{d_s^j} \mathop{\leq}\limits_{(d)} \sum_{i=1}^{N}\left(1+\frac{1}{\Delta}\right)^{\Delta}\notag\\
        &\leq \sum_{i=1}^{N}\mathrm{e}= N\mathrm{e}, \label{eq2-1}
    \end{align}
    where Ineq. $(c)$ adopts the inductive hypothesis and Ineq. $(d)$ adopts the fact of $d_s^j\leq\Delta$ when $(s,j)\in \Phi_{t}$ and $t\leq T$. Then, we have
    \begin{align}
        &\quad \sum_{i=1}^{N}p_{i}(t)\cdot\ell_{i}(t) = \sum_{i=1}^{N}p_{i}(t)\sum_{(s,j)\in \Phi_t}\hat{l}_{i}^{j}(s) \notag\\
        &= \sum_{(s,j)\in \Phi_t}\sum_{i=1}^{N}p_{i}(t)\cdot\hat{l}_{i}^{j}(s) \mathop{\leq}\limits_{(e)} \sum_{(s,j)\in \Phi_t}N\mathrm{e}\mathop{\leq}\limits_{(f)} MN\mathrm{e}, \label{eq2-2}
    \end{align}
    where Eq. (\ref{eq2-1}) brings about Ineq. $(e)$ and $\left\vert \Phi_t\right\vert \leq M$ brings about Ineq. $(f)$. According to this result and Lemma \ref{lem1}, we have
    \begin{align}
        &\quad p_i(t+1)\left(1-\eta'MN\mathrm{e}\right) \notag\\
        &\mathop{\leq}\limits_{(g)} p_i(t+1)\left(1-\eta'\sum_{i=1}^{N}p_{i}(t)\cdot\ell_{i}(t)\right) \notag\\
        &= p_i(t+1) - \eta'p_i(t+1)\sum_{i=1}^{N}p_{i}(t)\cdot\ell_{i}(t) \notag\\
        &\mathop{\leq}\limits_{(h)} p_i(t+1) - \left(p_i(t+1)-p_i(t)\right)= p_i(t)\notag
    \end{align}
    where Ineq. $(g)$ follows Eq. (\ref{eq2-2}) and Ineq. $(h)$ follows Lemma \ref{lem1}. By taking into account $\eta'\leq \frac{1}{MN\mathrm{e}(\Delta+1)}$, we can prove the inductive case
    \begin{align}
        &\quad p_i(t+1) \leq \frac{1}{1-\eta'MN\mathrm{e}}p_i(t) \notag\\
        &\leq \frac{1}{1-\frac{1}{\Delta+1}}p_i(t) = \left(1+\frac{1}{\Delta}\right)p_{i}(t). \notag
    \end{align}

    Finally, we finish the proof by combining the base case and the inductive case for a complete mathematical induction. %$\hfill\blacksquare$
\end{proof}

The two expectation terms in Lemma \ref{lem3} and Lemma \ref{lem4} are key components that constitute the transformed regret expression. For clarity of expression, we analyze the upper bound for these two terms separately in advance, and then utilize their consequences in the analysis of Theorem \ref{thm1} to obtain the final regret bound.

\begin{lem}\label{lem3}
    MUD-EXP3 satisfies the following inequality
    \begin{align}
        &\quad \mathbb{E}\left[\sum_{t=1}^{T}\sum_{(s,j)\in \Phi_t}\sum_{k=1}^{N}p_{k}(t)\cdot l_{k}^{j}(s) - \sum_{t=1}^{T}\sum_{j=1}^{M}l_{i}^{j}(t)\right] \notag\\
        &\leq \frac{\ln{N}}{\eta^{\prime}} + \frac{1}{2}\eta^{\prime}M^2TN\mathrm{e}. \notag
    \end{align}
\end{lem}
\begin{proof}
    According to the definition of $W_{i}(t)$, we can first calculate a lower bound of $\sum_{i=1}^{N}W_{i}(T)/\sum_{i=1}^{N}W_{i}(0)$ as follows.
    \begin{align}
        &\quad \frac{\sum_{i=1}^{N}W_{i}(T)}{\sum_{i=1}^{N}W_{i}(0)} = \frac{\sum_{i=1}^{N}\exp (-\eta^{\prime}\hat{L}_{i}^{T})}{\sum_{i=1}^{N}\exp (-\eta^{\prime}\hat{L}_{i}^{0})} \notag\\
        &\geq \frac{\max_{i\in\mathcal{N}} \exp\left(-\eta'\sum_{t=1}^{T}\sum_{(s,j)\in \Phi_t}\hat{l}_{i}^{j}(s)\right)}{N} \notag\\
        & \geq \frac{\exp\left(-\eta'\sum_{t=1}^{T}\sum_{(s,j)\in \Phi_t}\hat{l}_{i}^{j}(s)\right)}{N}. \notag        
    \end{align}

    Before we derive the corresponding upper bound, we first analyze the upper bound of $\sum_{i=1}^{N}W_{i}(t)/\sum_{i=1}^{N}W_{i}(t-1)$ and then telescope this over $T$. 
    \begin{align}
        &\quad \frac{\sum_{i=1}^{N}W_{i}(t)}{\sum_{i=1}^{N}W_{i}(t-1)} = \frac{\sum_{i=1}^{N}\exp (-\eta^{\prime}\hat{L}_{i}^{t})}{\sum_{i=1}^{N}\exp (-\eta^{\prime}\hat{L}_{i}^{t-1})}\notag\\
        &= \frac{\sum_{i=1}^{N}\exp (-\eta^{\prime}\hat{L}_{i}^{t-1})\exp(-\eta^{\prime}\ell_{i}^{t})}{\sum_{i=1}^{N}\exp (-\eta^{\prime}\hat{L}_{i}^{t-1})} \notag\\
        &= \sum_{i=1}^{N}p_i(t)\cdot \exp(-\eta^{\prime}\ell_{i}^{t}) \notag\\
        &= \sum_{i=1}^{N}p_i(t)\exp\left(-\eta^{\prime}\frac{1}{\lvert\Phi_t\rvert}\sum_{(s,j)\in \Phi_t}\lvert\Phi_t\rvert\cdot\hat{l}_{i}^{j}(s)\right) \notag\\
        & \mathop{\leq}\limits_{(a)} \sum_{i=1}^{N}p_i(t)\frac{1}{\lvert\Phi_t\rvert}\sum_{(s,j)\in \Phi_t}\exp\left(-\eta^{\prime}\lvert\Phi_t\rvert\cdot\hat{l}_{i}^{j}(s)\right) \notag\\
        & \mathop{\leq}\limits_{(b)} \sum_{i=1}^{N}p_i(t)\frac{1}{\lvert\Phi_t\rvert} \notag\\
        &\qquad \cdot\sum_{(s,j)\in \Phi_t}\left(1 - \eta^{\prime}\lvert\Phi_t\rvert\cdot\hat{l}_{i}^{j}(s) + \frac{1}{2}{\eta^{\prime}}^{2}{\lvert\Phi_t\rvert}^{2}\cdot{\hat{l}_{i}^{j}(s)}^{2}\right) \notag\\
        &= \sum_{i=1}^{N}p_i(t) \left(1 - \eta^{\prime}\sum_{(s,j)\in \Phi_t}\hat{l}_{i}^{j}(s) + \frac{1}{2}{\eta^{\prime}}^{2}\lvert\Phi_t\rvert\sum_{(s,j)\in \Phi_t}{\hat{l}_{i}^{j}(s)}^{2}\right) \notag\\
        &= 1 - \eta^{\prime}\sum_{(s,j)\in \Phi_t}\sum_{i=1}^{N}p_i(t)\cdot\hat{l}_{i}^{j}(s) \notag\\
        &\qquad + \frac{1}{2}{\eta^{\prime}}^{2}\lvert\Phi_t\rvert\sum_{(s,j)\in \Phi_t}\sum_{i=1}^{N}p_i(t)\cdot{\hat{l}_{i}^{j}(s)}^{2} \notag\\
        &\mathop{\leq}\limits_{(c)} \exp\left(- \eta^{\prime}\sum_{(s,j)\in \Phi_t}\sum_{i=1}^{N}p_i(t)\cdot\hat{l}_{i}^{j}(s) \right. \notag\\
        &\qquad \left. + \frac{1}{2}{\eta^{\prime}}^{2}\lvert\Phi_t\rvert\sum_{(s,j)\in \Phi_t}\sum_{i=1}^{N}p_i(t)\cdot{\hat{l}_{i}^{j}(s)}^{2}\right), \notag
    \end{align}
    where Ineq. $(a)$ follows Jensen's inequality \cite{jensen1906fonctions}, Ineq. $(b)$ follows $\exp(x)\leq 1+x+(1/2)x^2$ for $x\in \mathbb{R}$, and Ineq. $(c)$ follows $1+x\leq\exp(x)$ for $x\in \mathbb{R}$. Then, we use it to upper bound $\sum_{i=1}^{N}W_{i}(T)/\sum_{i=1}^{N}W_{i}(0)$.
    \begin{align}
        &\quad \frac{\sum_{i=1}^{N}W_{i}(T)}{\sum_{i=1}^{N}W_{i}(0)}= \prod_{t=1}^{T}\frac{\sum_{i=1}^{N}W_{i}(t)}{\sum_{i=1}^{N}W_{i}(t-1)} \notag\\
        &\leq \prod_{t=1}^{T}\exp\left(- \eta^{\prime}\sum_{(s,j)\in \Phi_t}\sum_{i=1}^{N}p_i(t)\cdot\hat{l}_{i}^{j}(s) \right. \notag\\
        &\qquad \left. + \frac{1}{2}{\eta^{\prime}}^{2}\lvert\Phi_t\rvert\sum_{(s,j)\in \Phi_t}\sum_{i=1}^{N}p_i(t)\cdot{\hat{l}_{i}^{j}(s)}^{2}\right) \notag\\
        &= \exp\left(- \eta^{\prime}\sum_{t=1}^{T}\sum_{(s,j)\in \Phi_t}\sum_{i=1}^{N}p_i(t)\cdot\hat{l}_{i}^{j}(s) \right. \notag\\
        &\qquad \left. + \frac{1}{2}{\eta^{\prime}}^{2}\sum_{t=1}^{T}\lvert\Phi_t\rvert\sum_{(s,j)\in \Phi_t}\sum_{i=1}^{N}p_i(t)\cdot{\hat{l}_{i}^{j}(s)}^{2}\right). \notag
    \end{align}

    Combining the lower bound and upper bound of $\sum_{i=1}^{N}W_{i}(T)/\sum_{i=1}^{N}W_{i}(0)$ establishes the following expression
    \begin{align}
        &\quad \frac{\exp\left(-\eta'\sum_{t=1}^{T}\sum_{(s,j)\in \Phi_t}\hat{l}_{i}^{j}(s)\right)}{N} \notag\\
        &\leq \exp\left(- \eta^{\prime}\sum_{t=1}^{T}\sum_{(s,j)\in \Phi_t}\sum_{i=1}^{N}p_i(t)\cdot\hat{l}_{i}^{j}(s) \right. \notag\\ 
        &\qquad \left. + \frac{1}{2}{\eta^{\prime}}^{2}\sum_{t=1}^{T}\lvert\Phi_t\rvert\sum_{(s,j)\in \Phi_t}\sum_{i=1}^{N}p_i(t)\cdot{\hat{l}_{i}^{j}(s)}^{2}\right). \notag
    \end{align}
    Take $\ln(\cdot)$ on both sides and do transposition.
    \begin{align}
        &\quad \eta^{\prime}\sum_{t=1}^{T}\sum_{(s,j)\in \Phi_t}\sum_{i=1}^{N}p_i(t)\cdot\hat{l}_{i}^{j}(s)- \eta'\sum_{t=1}^{T}\sum_{(s,j)\in \Phi_t}\hat{l}_{i}^{j}(s) \notag\\
        &\leq \ln{N} + \frac{1}{2}{\eta^{\prime}}^{2}\sum_{t=1}^{T}\lvert\Phi_t\rvert\sum_{(s,j)\in \Phi_t}\sum_{i=1}^{N}p_i(t)\cdot{\hat{l}_{i}^{j}(s)}^{2}. \label{eq3-1}
    \end{align}
    Before moving on, we conduct an analysis of the expectations on both sides of Ineq. (\ref{eq3-1}), respectively. We first deal with the left part.
    \begin{align}
        &\quad\mathbb{E}\left[\eta^{\prime}\sum_{t=1}^{T}\sum_{(s,j)\in \Phi_t}\sum_{i=1}^{N}p_i(t)\cdot\hat{l}_{i}^{j}(s) - \eta'\sum_{t=1}^{T}\sum_{(s,j)\in \Phi_t}\hat{l}_{i}^{j}(s)\right] \notag\\
        &\mathop{=}\limits_{(d)} \eta^{\prime}\mathbb{E}\left[\sum_{t=1}^{T}\sum_{(s,j)\in \Phi_t}\sum_{i=1}^{N}p_i(t)\cdot\mathbb{E}\left[\hat{l}_{i}^{j}(s)\vert \mathcal{F}_{t-1}\right] \right. \notag\\ 
        &\qquad \left. - \sum_{t=1}^{T}\sum_{(s,j)\in \Phi_t}\mathbb{E}\left[\hat{l}_{i}^{j}(s)\vert \mathcal{F}_{t-1}\right]\right] \notag\\
        &\mathop{=}\limits_{(e)} \eta^{\prime}\mathbb{E}\left[\sum_{t=1}^{T}\sum_{(s,j)\in \Phi_t}\sum_{i=1}^{N}p_i(t)\cdot l_{i}^{j}(s) - \sum_{t=1}^{T}\sum_{(s,j)\in \Phi_t}l_{i}^{j}(s)\right],\label{eq3-2}
    \end{align}
    where Eq. $(d)$ uses $p_i(t)\in \mathcal{F}_{t-1}$ and Eq. $(e)$ uses $p_i(s)\in \mathcal{F}_{t-1}$ together with the fact that $\hat{l}_{i}^{j}(s)$ is $l_{i}^{j}(s)/p_i(s)$ with probability $p_i(s)$ and zero otherwise. Then, the same operation is carried out for the right part.
    \begin{align}
        &\quad \mathbb{E}\left[\ln{N} + \frac{1}{2}{\eta^{\prime}}^{2}\sum_{t=1}^{T}\lvert\Phi_t\rvert\sum_{(s,j)\in \Phi_t}\sum_{i=1}^{N}p_i(t)\cdot{\hat{l}_{i}^{j}(s)}^{2}\right] \notag\\
        &= \ln{N} \notag\\
        &\qquad + \frac{1}{2}{\eta^{\prime}}^{2}\mathbb{E}\left[\sum_{t=1}^{T}\lvert\Phi_t\rvert\sum_{(s,j)\in \Phi_t}\sum_{i=1}^{N}p_i(t)\cdot\mathbb{E}\left[{\hat{l}_{i}^{j}(s)}^{2}\vert\mathcal{F}_{t-1}\right]\right] \notag\\
        &= \ln{N} + \frac{1}{2}{\eta^{\prime}}^{2}\mathbb{E}\left[\sum_{t=1}^{T}\lvert\Phi_t\rvert\sum_{(s,j)\in \Phi_t}\sum_{i=1}^{N}\frac{p_i(t)}{p_i(s)}{l_{i}^{j}(s)}^{2}\right] \notag\\
        &\mathop{\leq}\limits_{(f)} \ln{N} + \frac{1}{2}{\eta^{\prime}}^{2}M\mathbb{E}\left[\sum_{t=1}^{T}\sum_{(s,j)\in \Phi_t}\sum_{i=1}^{N}\frac{p_i(t)}{p_i(s)}\right]\notag\\
        &\mathop{\leq}\limits_{(g)} \ln{N} + \frac{1}{2}{\eta^{\prime}}^{2}M\mathbb{E}\left[\sum_{t=1}^{T}\sum_{(s,j)\in \Phi_t}\sum_{i=1}^{N}\left(1+\frac{1}{\Delta}\right)^{d_s^j}\right]\notag\\
        &\leq \ln{N} + \frac{1}{2}{\eta^{\prime}}^{2}M\mathbb{E}\left[\sum_{t=1}^{T}\sum_{(s,j)\in \Phi_t}\sum_{i=1}^{N}\Bigg(1 \right. \notag\\
        &\qquad \left.\left. +\frac{1}{\sum_{t'=1}^{T}\sum_{({s'},{j'})\in \Phi_{t'}}d_{s'}^{j'}}\right)^{\Delta}\right]\notag\\
        &\leq \ln{N} + \frac{1}{2}{\eta^{\prime}}^{2}M^2TN\mathrm{e},\label{eq3-3}
    \end{align}
    where Ineq. $(f)$ is due to $l_{i}^{j}(s)\leq 1$ and $\lvert\Phi_t\rvert\leq M$, and Ineq. $(g)$ uses Lemma \ref{lem2}. By substituting Eq. (\ref{eq3-2}) and Eq. (\ref{eq3-3}) into Eq. (\ref{eq3-1}), we have
    \begin{align}
        &\quad \eta^{\prime}\mathbb{E}\left[\sum_{t=1}^{T}\sum_{(s,j)\in \Phi_t}\sum_{i=1}^{N}p_i(t)\cdot l_{i}^{j}(s) - \sum_{t=1}^{T}\sum_{(s,j)\in \Phi_t}l_{i}^{j}(s)\right] \notag\\
        &\leq \ln{N} + \frac{1}{2}{\eta^{\prime}}^{2}M^2TN\mathrm{e}. \notag
    \end{align}
    Thus, the eventual result is presented as follows.
    \begin{align}
        &\quad \mathbb{E}\left[\sum_{t=1}^{T}\sum_{(s,j)\in \Phi_t}\sum_{k=1}^{N}p_{k}(t)\cdot l_{k}^{j}(s) - \sum_{t=1}^{T}\sum_{j=1}^{M}l_{i}^{j}(t)\right] \notag\\
        &\leq \mathbb{E}\left[\sum_{t=1}^{T}\sum_{(s,j)\in \Phi_t}\sum_{i=1}^{N}p_i(t)\cdot l_{i}^{j}(s) - \sum_{t=1}^{T}\sum_{(s,j)\in \Phi_t}l_{i}^{j}(s)\right] \notag\\
        &\leq \frac{\ln{N}}{\eta^{\prime}} + \frac{1}{2}\eta^{\prime}M^2TN\mathrm{e}. \notag
    \end{align}
    The lemma has been proven. %$\hfill\blacksquare$
\end{proof}

\begin{lem}\label{lem4}
    MUD-EXP3 satisfies the following inequality
    \begin{align}
        &\mathbb{E}\left[\sum_{t=1}^{T}\sum_{(s,j)\in \Phi_t}\sum_{k=1}^{N}p_{k}(s)\cdot l_{k}^{j}(s) \right.\notag\\
        &\qquad \left. - \sum_{t=1}^{T}\sum_{(s,j)\in \Phi_t}\sum_{k=1}^{N}p_{k}(t)\cdot l_{k}^{j}(s)\right] \leq 2\eta'M\sum_{t=1}^{T}\sum_{(s,j)\in \Phi_t}d_s^j. \notag
    \end{align}
\end{lem}
\begin{proof}
    Due to the space constraint, the proof of Lemma \ref{thm1} is relegated to Appendix \ref{lem4_proof}.
\end{proof}

\begin{thm}\label{thm1}
    For any arm $i\in\mathcal{N}$, MUD-EXP3 guarantees the upper bound for $\mathcal{R}_i$ is shown as 
    \begin{equation}
        \mathcal{R}_i \leq \frac{\ln{N}}{\eta^{\prime}} + \frac{1}{2}\eta^{\prime}M^2TN\mathrm{e} + 2\eta'M\sum_{t=1}^{T}\sum_{(s,j)\in \Phi_t}d_s^j + \lvert\Omega\rvert, \notag
    \end{equation}
    which implies the same regret upper bound as follows:
    \begin{equation}
        \mathcal{R} \leq \frac{\ln{N}}{\eta^{\prime}} + \frac{1}{2}\eta^{\prime}M^2TN\mathrm{e} + 2\eta'M\sum_{t=1}^{T}\sum_{(s,j)\in \Phi_t}d_s^j + \lvert\Omega\rvert. \notag
    \end{equation}
    Specially, for the known $T$ and $\sum_{t=1}^{T}\sum_{(s,j)\in \Phi_t}d_s^j$, if 
    \begin{align}
        \eta&= \sqrt{\frac{\ln{N}}{M(TMN\mathrm{e}+4\sum_{t=1}^{T}\sum_{(s,j)\in \Phi_t}d_s^j)}} \notag\\
        &\leq \frac{1}{MN\mathrm{e}(\sum_{t=1}^{T}\sum_{(s,j)\in \Phi_t}d_s^j+1)}, \notag
    \end{align}
    we have
    \begin{align}
        \mathcal{R} &\leq \mathcal{O}\left(\sqrt{M\ln{N}(TMN\mathrm{e}+4\sum_{t=1}^{T}\sum_{(s,j)\in \Phi_t}d_s^j})\right) \notag\\
        &\leq \mathcal{O}\left(\sqrt{TM^2\ln{N}\left(N\mathrm{e}+4d_{max}\right)}\right). \notag
    \end{align}
\end{thm}
\begin{proof}
    Due to the space constraint, the proof of Theorem \ref{thm1} is relegated to Appendix \ref{thm1_proof}.
\end{proof}

\subsection{Regret of AMUD-EXP3}\label{Reg of AMUD-EXP3}
Based on the analysis of MUD-EXP3, we further prove the regret upper bound of AMUD-EXP3 in Theorem \ref{thm2} supported by Lemma \ref{lem5}-\ref{lem7} in the remaining section. The proof in this part refers to the doubling trick frameworks proposed by \cite{cesa1997use} and \cite{bistritz2019online}. By applying Theorem \ref{thm1} on the regret of each epoch, we bound the epoch regret as
\begin{equation}\label{epc-rgt}
    \mathcal{R}_{\varepsilon}\leq \frac{\ln{N}}{\eta_{\varepsilon}} + \frac{1}{2}\eta_{\varepsilon}M^2\lvert \mathcal{T}_{\varepsilon}\rvert N\mathrm{e} + 2\eta_{\varepsilon}M\sum_{t\in\mathcal{T}_{\varepsilon}}\sum_{(s,j)\in \Phi_t}d_s^j + \lvert\Omega_{\varepsilon}\rvert.
\end{equation}

\begin{lem}\label{lem5}
    For AMUD-EXP3, the sum of the missing count $\sum_{t\in\mathcal{T}_{\varepsilon}}V_t$ satisfies the following inequalities:
    \begin{equation*}
        \sum_{t\in\mathcal{T}_{\varepsilon}}\sum_{(s,j)\in \Phi_t}d_s^j \leq \sum_{t\in\mathcal{T}_{\varepsilon}}V_t \leq \left(2^{\varepsilon-1}+\frac{1}{\varepsilon}\right)M.
    \end{equation*}
\end{lem}

\begin{proof}
    For the left inequality, the round-user index pair $(t,j)$ of which $t\in \mathcal{T}_\varepsilon$ and $t\notin \Omega_{\varepsilon}$ contributes $d_t^j$ to $\sum_{t\in\mathcal{T}_{\varepsilon}}V_t$, because it adds one to $\sum_{t\in\mathcal{T}_{\varepsilon}}V_t$ each round from $t+1$ until $t+d_t^j$. Thus, $\sum_{t\in\mathcal{T}_{\varepsilon}}\sum_{(s,j)\in \Phi_t}d_s^j \leq \sum_{t\in\mathcal{T}_{\varepsilon}}V_t$.

    For the right inequality, we prove it by using contradiction. Assume $\sum_{t\in\mathcal{T}_{\varepsilon}}V_t > \left(2^{\varepsilon-1}+\frac{1}{\varepsilon}\right)M$, then we have
    \begin{align}
        \sum_{t=1}^{T_{\varepsilon}}V_t &= \sum_{r=1}^{\varepsilon}\sum_{t=T_{r-1}+1}^{T_r}V_t > \sum_{r=1}^{\varepsilon}\left(2^{r-1}+\frac{1}{r}\right)M \notag\\
        &\geq \sum_{r=1}^{\varepsilon}\left(2^{r-1}+\frac{1}{\varepsilon}\right)M = 2^{\varepsilon}M,
    \end{align}
    which contradicts the definition of $\mathcal{T}_\varepsilon$ where $\sum_{t=1}^{T_{\varepsilon}}V_t$ should be no more than $2^{\varepsilon}M$. Hence we have proved $\sum_{t\in\mathcal{T}_{\varepsilon}}V_t \leq \left(2^{\varepsilon-1}+\frac{1}{\varepsilon}\right)M$. %$\hfill\blacksquare$
\end{proof}

\begin{lem}\label{lem6}
    For epoch $\varepsilon$, the cardinal number of $\Omega_{\varepsilon}$ can be upper bound as $\lvert\Omega_{\varepsilon}\rvert \leq 2^{\frac{\varepsilon}{2}}\cdot 2M$.
\end{lem}
\begin{proof}
    We consider a situation that maximizes $\lvert\Omega_{\varepsilon}\rvert$ when setting the delay of any user $j$ at round $t$ as $d_t^j=T_{\varepsilon}-t+1$. Denoting the maximal $\lvert\Omega_{\varepsilon}\rvert$ under this circumstance as $\lvert\Omega_{\varepsilon}\rvert_{max}$, therein we have $\lvert\Omega_{\varepsilon}\rvert_{max}=M\left(T_{\varepsilon}-T_{\varepsilon-1}\right)$. Then, we equivalently represent $\sum_{t\in\mathcal{T}_{\varepsilon}}V_t$ via introducing $\lvert\Omega_{\varepsilon}\rvert_{max}$.
    \begin{align}
        \sum_{t\in\mathcal{T}_{\varepsilon}}V_t &= \sum_{t\in\mathcal{T}_{\varepsilon}}i\cdot M = M\cdot\sum_{i=1}^{\lvert\Omega_{\varepsilon}\rvert_{max}/M}i \notag\\
        &= \frac{1}{2}\lvert\Omega_{\varepsilon}\rvert_{max}\left(1+\lvert\Omega_{\varepsilon}\rvert_{max}/M\right) \mathop{\leq}\limits_{(a)} \left(2^{\varepsilon-1}+\frac{1}{\varepsilon}\right)M, \notag
    \end{align}
    where Ineq. $(a)$ follows Lemma \ref{lem5}. Next, we upper bound $\lvert\Omega_{\varepsilon}\rvert_{max}$ from Ineq. $(a)$
    \begin{align}
        &\qquad \lvert\Omega_{\varepsilon}\rvert_{max}^2+M\lvert\Omega_{\varepsilon}\rvert_{max} \leq \left(2^{\varepsilon}+\frac{2}{\varepsilon}\right)M^2 \notag\\
        &\implies \left(\lvert\Omega_{\varepsilon}\rvert_{max}+\frac{1}{2}M\right)^2 \leq \left(2^{\varepsilon}+\frac{2}{\varepsilon}+\frac{1}{4}\right)M^2 \notag\\
        &\implies \lvert\Omega_{\varepsilon}\rvert_{max}+\frac{1}{2}M \leq \sqrt{2^{\varepsilon}+\frac{2}{\varepsilon}+\frac{1}{4}}M \notag\\
        &\implies \lvert\Omega_{\varepsilon}\rvert_{max} \leq \left(\sqrt{2^{\varepsilon}+\frac{2}{\varepsilon}+\frac{1}{4}}-\frac{1}{2}\right)M \leq 2^{\frac{\varepsilon}{2}}\cdot 2M. \notag
    \end{align}
    Thus, there exist $\lvert\Omega_{\varepsilon}\rvert \leq \lvert\Omega_{\varepsilon}\rvert_{max} \leq 2^{\frac{\varepsilon}{2}}\cdot 2M$. %$\hfill\blacksquare$
\end{proof}

\begin{lem}\label{lem7}
    Let $E$ be the index of the final epoch. Then $2^{E-1}\leq\frac{1}{M}\sum_{t=1}^{T}\sum_{j=1}^{M}d_t^j$ and $\sum_{\varepsilon=1}^{E}\lvert\mathcal{T}_\varepsilon\rvert 2^{-\frac{\varepsilon}{2}}\leq 5\sqrt{T}$ hold.
\end{lem}
\begin{proof}
    According to the definition of $\mathcal{T}_\varepsilon$, we prove the first inequality as follows.
    \begin{align}
        2^{E-1} &\leq \frac{1}{M}\sum_{t=1}^{T}V_t = \frac{1}{M}\sum_{t=1}^{T}\sum_{j=1}^{M}\min\left\{d_t^j,T-t+1\right\} \notag\\
        &\leq \frac{1}{M}\sum_{t=1}^{T}\sum_{j=1}^{M}d_t^j. \notag
    \end{align}
    Note that $d_t^j\geq 1$ in the problem setting indicates $\lvert\mathcal{T}_\varepsilon\rvert \leq 2^{\varepsilon}$. The second inequality subject to $\sum_{\varepsilon=1}^{E}\lvert\mathcal{T}_\varepsilon\rvert=T$ is maximized when there are $\left\lceil \log_{2}T\right\rceil$ epochs with length of $2^{\varepsilon}$ to epoch $\varepsilon$
    \begin{align}
        \sum_{\varepsilon=1}^{E}\lvert\mathcal{T}_\varepsilon\rvert 2^{-\frac{\varepsilon}{2}}
        \leq \sum_{\varepsilon=1}^{\left\lceil \log_{2}T\right\rceil}2^{\varepsilon}\cdot 2^{-\frac{\varepsilon}{2}} \leq \sqrt{2}\frac{2^{\frac{\left\lceil \log_{2}T\right\rceil}{2}}-1}{\sqrt{2}-1} \leq 5\sqrt{T}. \notag
    \end{align}
    Thus, we finish the proof. These two inequalities will be used to eliminate the association with $E$ in the regret upper bound. %$\hfill\blacksquare$
\end{proof}

\begin{thm}\label{thm2}
    For any arm $i\in\mathcal{N}$, AMUD-EXP3 guarantees the upper bound for $\mathcal{R}_i$ as shown below.
    \begin{equation}
        \mathcal{R}_i \leq \left(11\sqrt{M\ln{N}}+7\sqrt{M}\right)\sqrt{\sum_{t=1}^{T}\sum_{j=1}^{M}d_t^j} + \frac{5}{2}MN\mathrm{e}\sqrt{T\ln{N}}. \notag
    \end{equation}
\end{thm}
\begin{proof}
    Due to the space constraint, the proof of Theorem \ref{thm2} is relegated to Appendix \ref{thm2_proof}.
\end{proof}

\section{Applications}
Many of the real-world scenarios can be modeled with our framework and solved by adopting our proposed algorithms MUD-EXP3 and AMUD-EXP3. In this section, we present some related applications and introduce how to deploy our algorithms in detail.

\subsection{Caching Optimization in Mobile Edge Computing}
The emerging 5G communication technologies have been bringing significantly increasing connected mobile devices and data traffic. For example, YouTube consumes over $440000$ terabytes of data daily, while blablabla. To reduce the backhaul transmission delay, mobile edge computing (MEC) is introduced, which caches content at the edge servers for the high quality-of-experience (QoE) of users. Due to the limitation of cache capacity, for better performance, edge servers should proactively cache the most popular content rather than reactively wait to receive the request sent by users. Some existing works \cite{borst2010distributed,poularakis2016exploiting} assume the content popularity is known in advance, which in practice, however, is very unlikely to be true. Moreover, the served user group tends to be dynamic with users coming and leaving, leading to non-stationary content popularity. In addition, to improve caching performance, collecting user preference feedback (eg. a user like it or not) is necessary for QoE-oriented service \cite{chen2018caching,zhang2016clustered}, which can hardly guarantee every user giving instant feedback and in turn brings about the delayed feedback issue.

In this case, we can apply our proposed online learning framework of multi-user delayed-feedback adversarial bandit to this caching optimization problem. Specifically, each combination of content can be seen as an arm while the content service provider is regarded as the player. In each discrete time slot, the content service provider chooses one combination of content of cache, the corresponding feedback from different users might be received after different periods due to individual behavior. For example, some users might be occupied by other affairs before giving feedback, some users might start to consume the content after several slots of time, and some users might never provide feedback. By employing our proposed algorithm MUD-EXP3 and AMUD-EXP3, the service provider can continuously collect feedback for previous decisions and adaptively learn the latest content popularity to maximize the cumulative total user experience. 

\subsection{Long-Term Traffic Management}
Nowadays the number of vehicles has increased dramatically. For densely populated cities, failing to handle the traffic flow would cause problems such as traffic jamming and road accidents. Hence traffic management is crucial to road safety and traffic flow efficiency. As an important part of a smart city, intelligent traffic management is empowered by machine-learning techniques in terms of predicting optimum routes, reducing traffic congestion, etc. Specifically, the applications of the Internet of Things (IoT) contribute to a great amount of traffic data which can be fed into the trained machine learning models to make predictions regarding the levels of traffic congestion in a particular area of a city. Based on those predictions, the downstream policy model gives suggestions to the traffic management authorities to deploy management policy (eg. traffic lights management). For a new system in an area, there are initially several candidate policy models with different architectures or parameters and unknown in-place performance. The objective is to jointly decrease the average waiting time on the road and increase the satisfaction of residents.

The traffic system we describe above faces an extremely sophisticated situation. With the presence of model prediction error, individual behavior randomness, unpredictable environmental factors, etc, it is unlikely to have the best policy model for all the time. Thus, we can regard it as an adversarial bandit problem and adopt our proposed algorithms to make an online optimization. The policy models can be considered as arms, and the traffic management authority acts as the player to carry out the policy decision and collect residents' satisfactory feedback. Normally, the satisfaction feedback from residents suffers delays to different extents since it is unrealistic to force residents to react instantly.

\section{Numerical Evaluation}\label{sec6}
In this section, we evaluate the performance of the proposed algorithms by comparing them with other state-of-the-art baselines. All of our simulations are programmed by Python 3.11 and run on a Windows 10 platform with CPU of Intel Core i7-11700 and 32GB RAM. The source code can be found at \url{https://github.com/Chubbro/MUD-EXP3}.

\begin{figure*}[!t]
	\centering
    \subfigure{\includegraphics[width=0.328\linewidth]{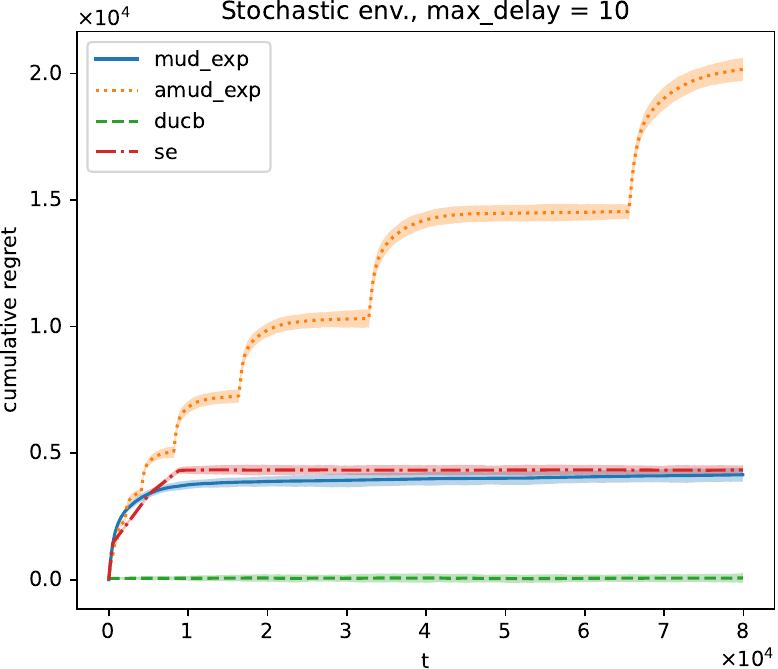}}
    \subfigure{\includegraphics[width=0.323\linewidth]{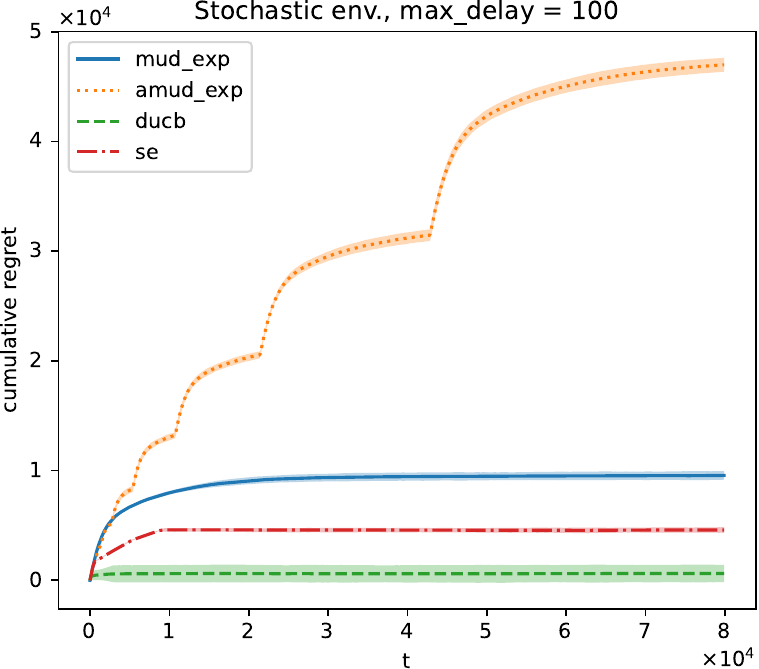}}
    \subfigure{\includegraphics[width=0.332\linewidth]{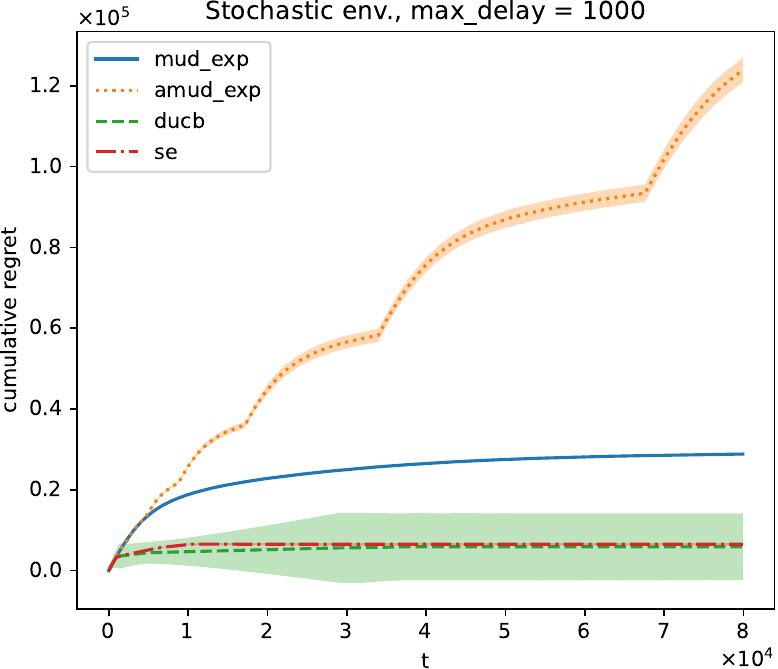}}
	\caption{The cumulative regret under a stochastic bandit environment.}
	\label{fig_regret}
\end{figure*}

\begin{figure*}[!t]
	\centering
    \subfigure{\includegraphics[width=0.328\linewidth]{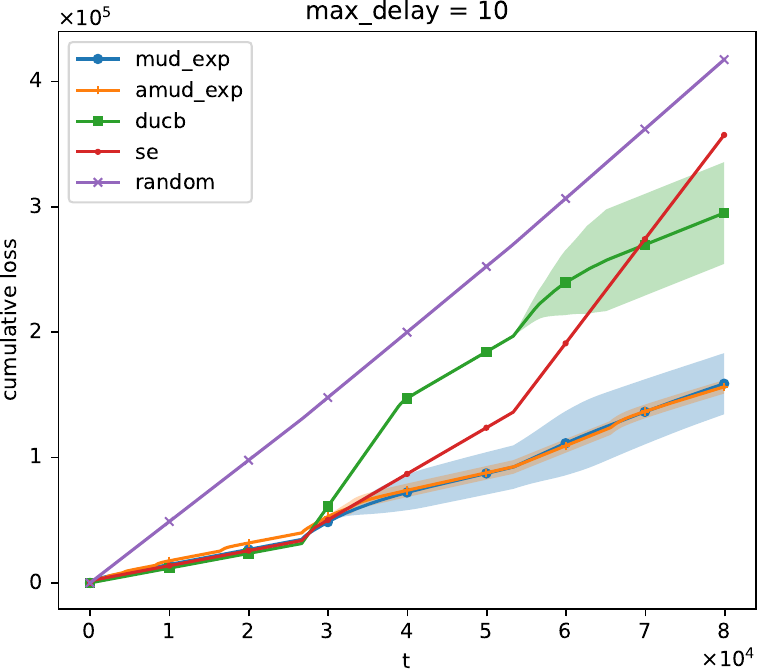}}
    \subfigure{\includegraphics[width=0.328\linewidth]{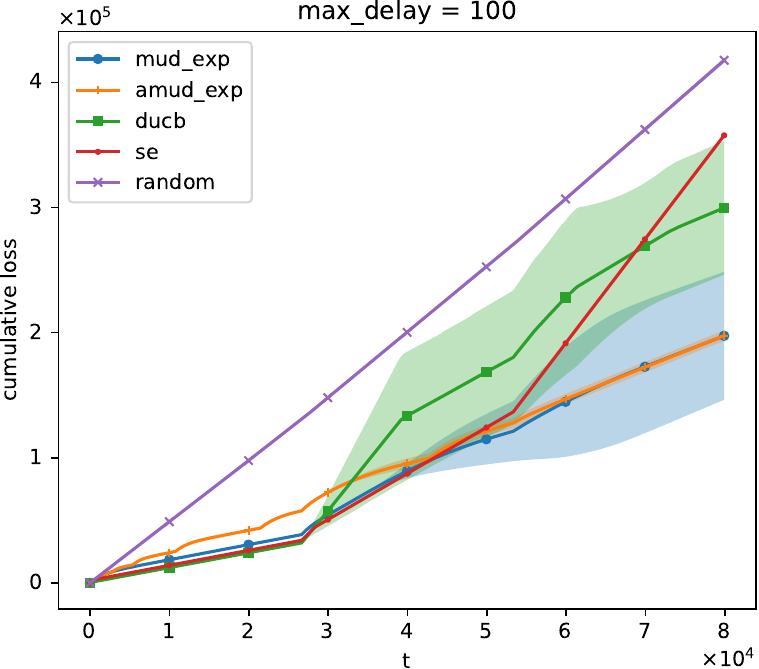}}
    \subfigure{\includegraphics[width=0.328\linewidth]{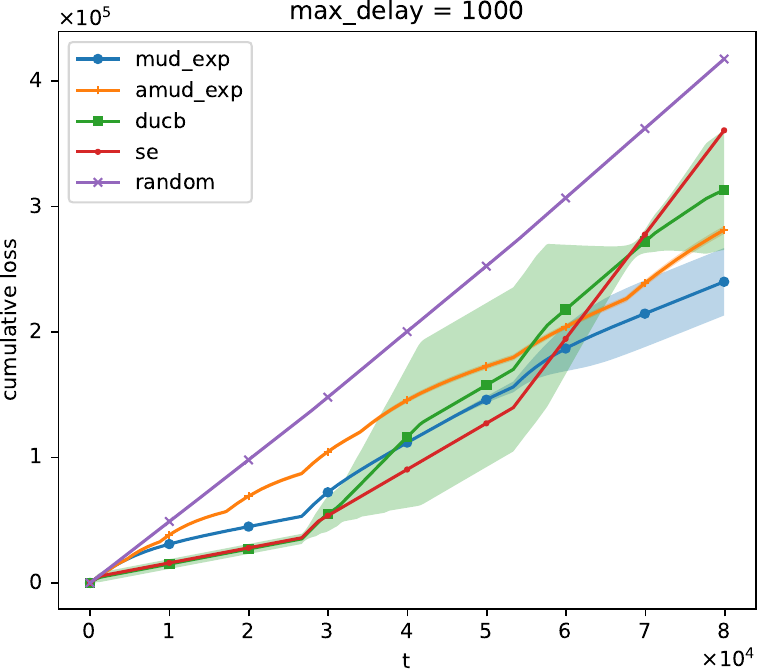}}
	\caption{The cumulative loss with varying maximum delays.}
	\label{fig_loss}
\end{figure*}

\subsection{Adversarial Bandits Setup}
This paper's basic environment of adversarial bandits consists of three different loss distributions for each arm, simulating the distribution transition pervasive in the real world. The number of changes in distribution for one arm is denoted as \textit{tran\_num}, and the interval between two changes is uniformly set as $1/\textit{tran\_num}$ of $T$. The distributions are all generated by truncated Gaussian of range $[0,1]$ with random means in the range $[0,1]$ and random standard deviations in the range $[0.1,0.2]$. We assign $N=10$, $M=10$ and $\textit{tran\_num}=3$ by default and vary the maximum delay $d_{max}=[10,100,1000]$.

The evaluation starts with the results of a stochastic bandit for reference, while the focus is on the subsequent adversarial bandits. We first evaluate the cumulative regrets of candidate algorithms with respect to round $t$ in Figure \ref{fig_regret} under a stochastic bandit environment where the distribution of loss for each arm is fixed all along, then change the environment into our default adversarial one to assess the performance in Figure \ref{fig_loss}. Note that due to the limitation of the oracle setting, i.e. choosing a fixed arm throughout the procedure tends to bring about inferior performance in adversarial bandits, the results are primarily presented with a loss indicator rather than a regret one. Next, in order to investigate the influence of evolving distributions of loss on performance, we compare the total loss among three superior algorithms under adversarial bandit environments with diverse \textit{tran\_num} values in Figure \ref{fig_bar1}-\ref{fig_bar3}, and demonstrate the specific tendency of loss when further increasing \textit{tran\_num} in Figure \ref{fig_num_d10}-\ref{fig_num_d100}. In addition, we also qualitatively show the impacts of arm number $N$ and user number $M$ on the performance of approaches in Figure \ref{fig_n_m}. Note that the shade regions in those figures indicate the variation range with respect to two standard deviations.

\subsection{Benchmarks}
In the experiment, to demonstrate the efficacy, our proposed algorithm is compared with the following benchmarks:
\begin{itemize}
        \item \textit{oracle}: the Oracle algorithm has the prior knowledge of real expected loss for each arm. In every round, the oracle selects the fixed arm with the minimal total expected loss.
        \item \textit{mud}: The first proposed algorithm MUD-EXP3 shown as Algorithm \ref{MUD-EPX3} in this paper.
        \item \textit{amud}: The second proposed algorithm AMUD-EXP3 shown as Algorithm \ref{AMUD-EPX3} in this paper.
        \item \textit{ducb} \cite{vernade2017stochastic}: A state-of-the-art algorithm that adjusts the classic UCB framework to a delayed feedback setting. Here we modified it to adapt to the problem of multiple users.
        \item \textit{se} \cite{lancewicki2021stochastic}: A state-of-the-art algorithm under delayed UCB framework with successive elimination of arms. Here we modified it to adapt to the problem of multiple users.
        \item \textit{random}: A lazy algorithm that selects an arm randomly at each round.
    \end{itemize}

\subsection{Result Analysis}

\begin{figure}[!t]
	\centering
    \includegraphics[width=\linewidth]{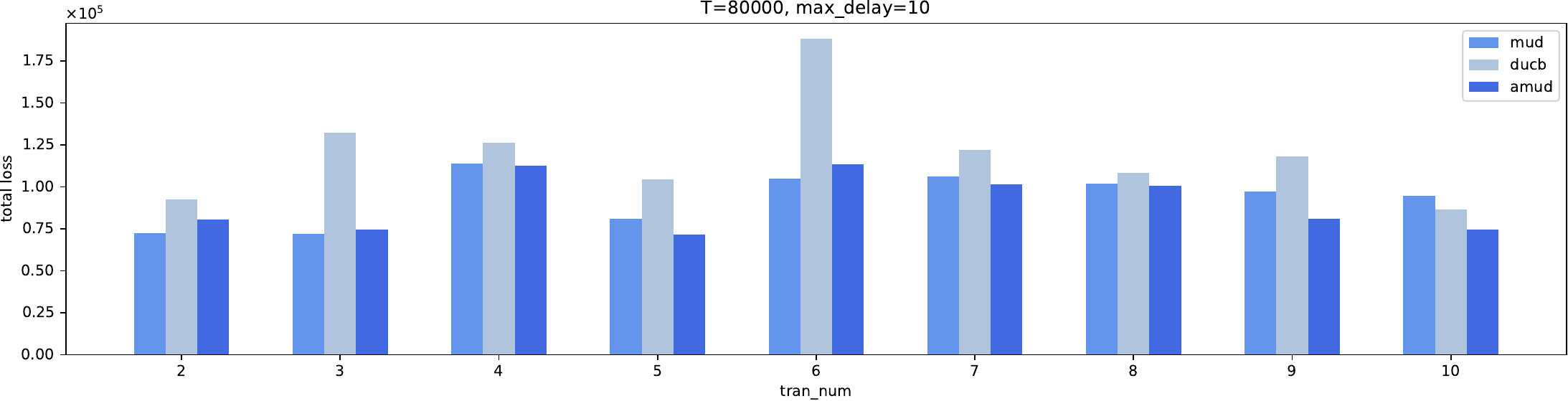}
	\caption{\centering{The total loss in adversarial bandits with varying \textit{tran\_num} and $d_{max}=10$.}}
	\label{fig_bar1}
\end{figure}

\begin{figure}[!t]
	\centering
    \includegraphics[width=\linewidth]{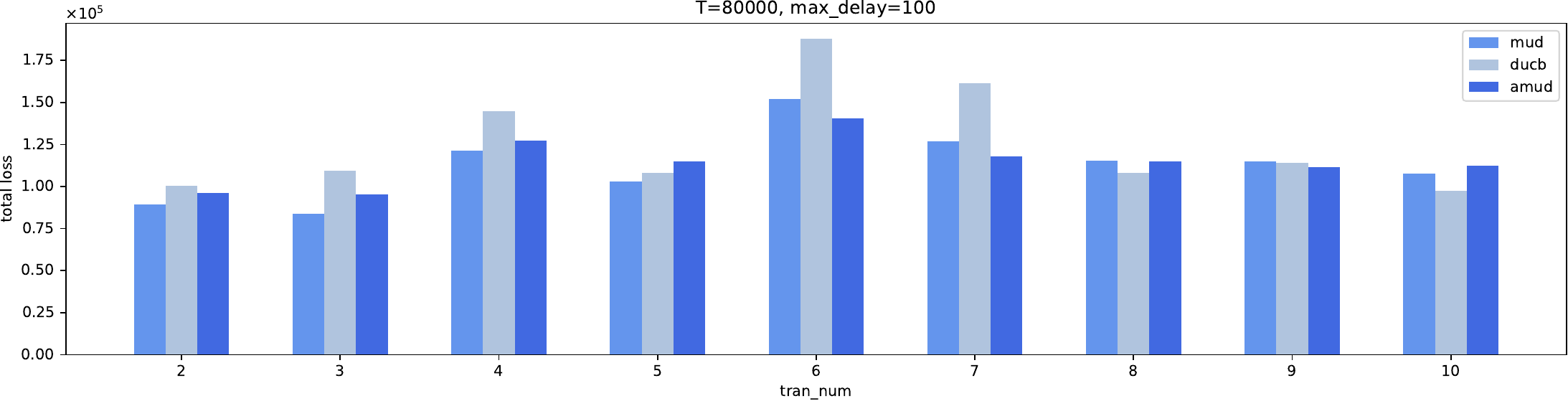}
	\caption{\centering{The total loss in adversarial bandits with varying \textit{tran\_num} and $d_{max}=100$.}}
	\label{fig_bar2}
\end{figure}

\begin{figure}[!t]
	\centering
    \includegraphics[width=\linewidth]{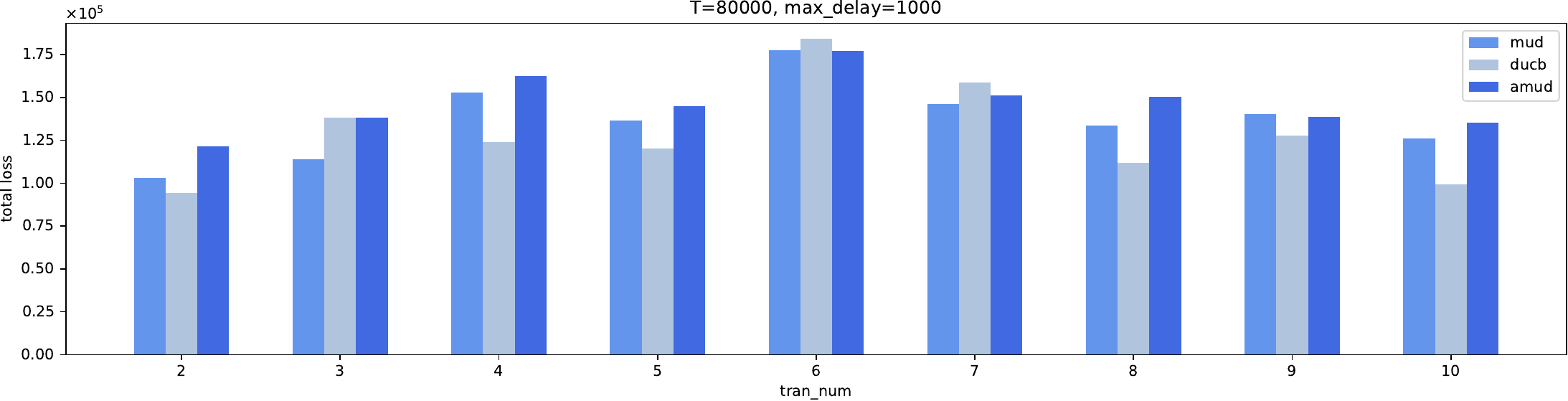}
	\caption{\centering{The total loss in adversarial bandits with varying \textit{tran\_num} and $d_{max}=1000$.}}
	\label{fig_bar3}
\end{figure}

\begin{figure}[!t]
	\centering
    \includegraphics[width=\linewidth]{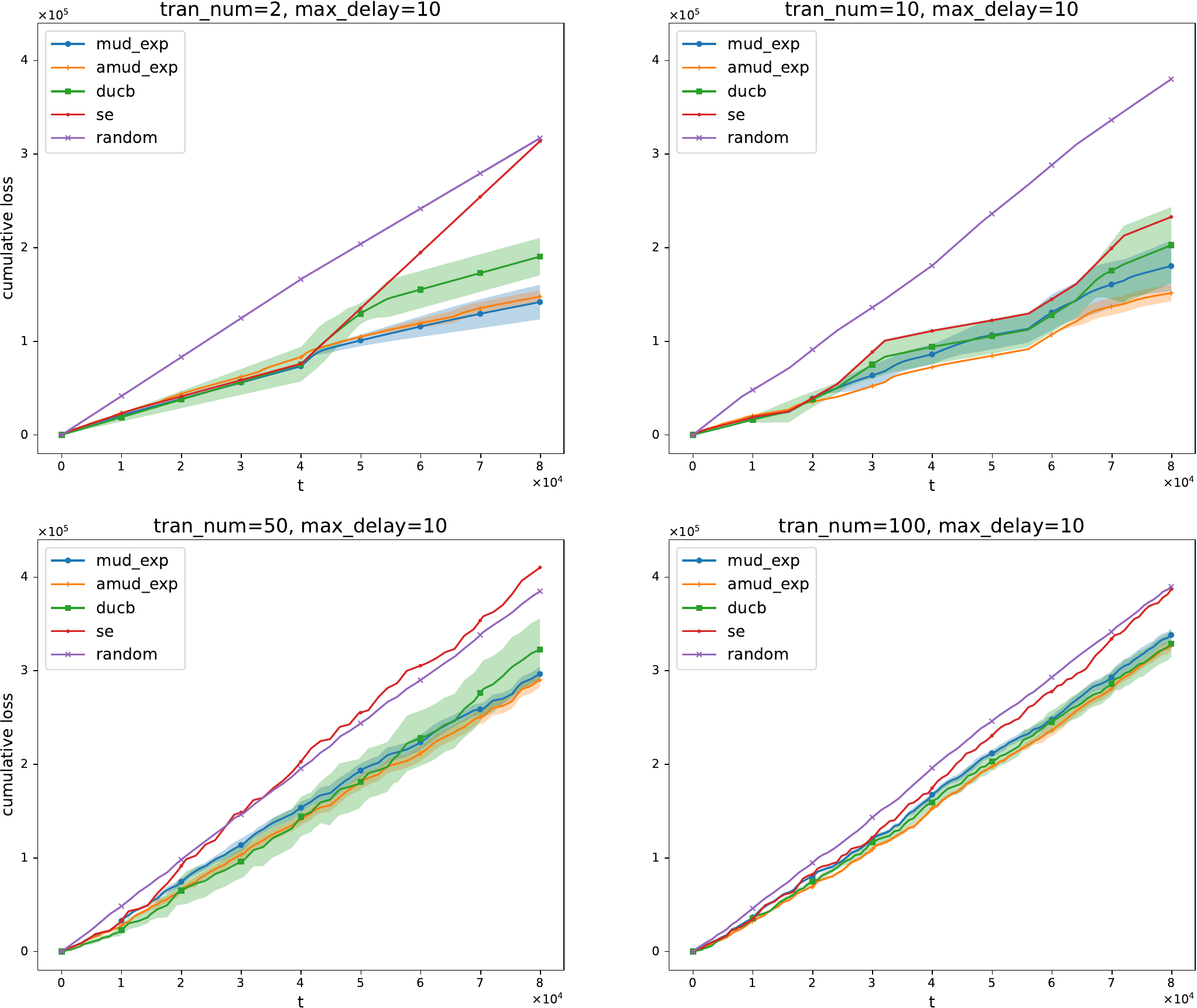}
	\caption{\centering{The cumulative loss in adversarial bandits with $d_{max}=10$ and varying \textit{tran\_num}.}}
	\label{fig_num_d10}
\end{figure}

\begin{figure}[!t]
	\centering
    \includegraphics[width=\linewidth]{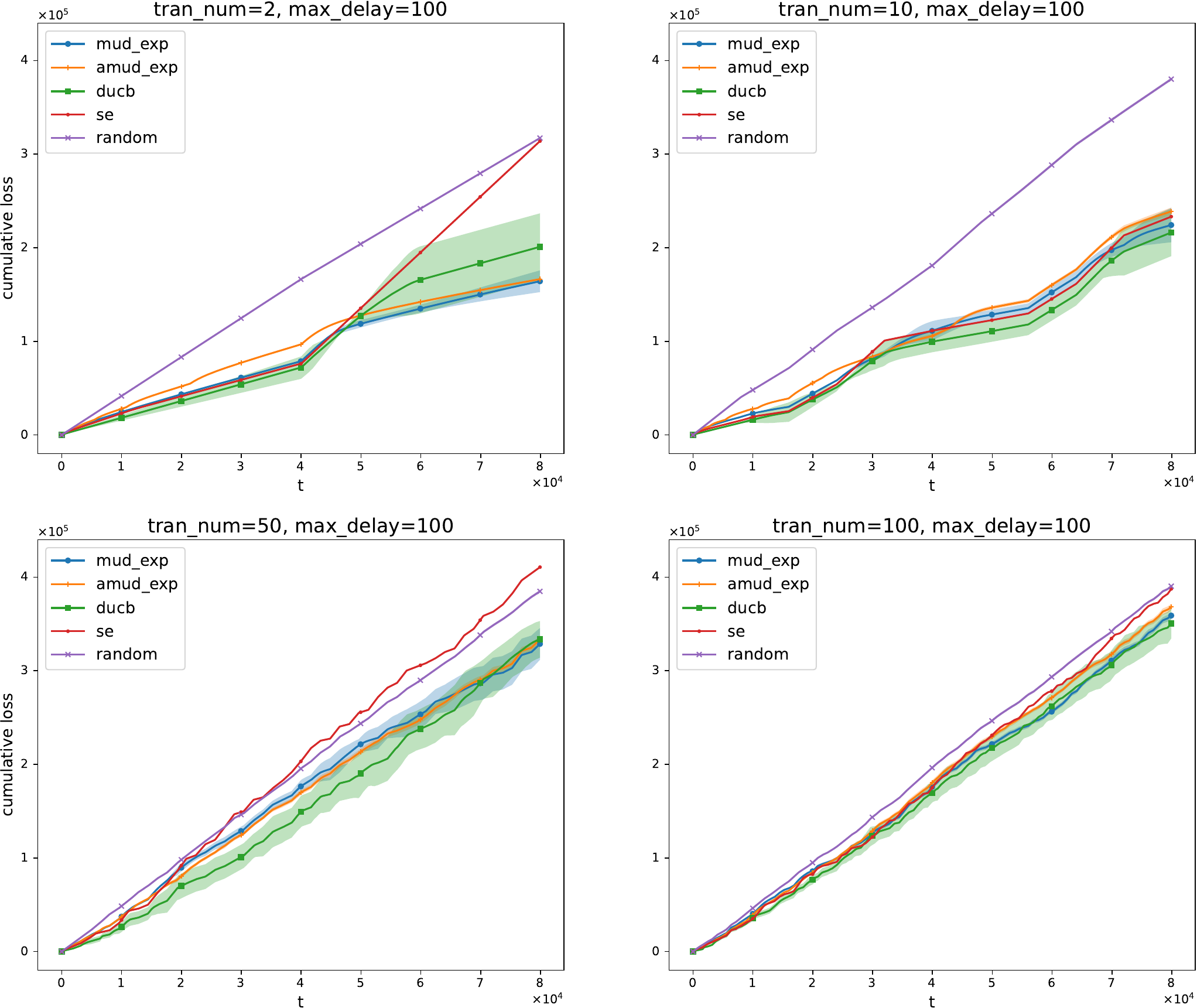}
	\caption{\centering{The cumulative loss in adversarial bandits with $d_{max}=100$ and varying \textit{tran\_num}.}}
	\label{fig_num_d100}
\end{figure}

\begin{figure}[!t]
	\centering
    \includegraphics[width=\linewidth]{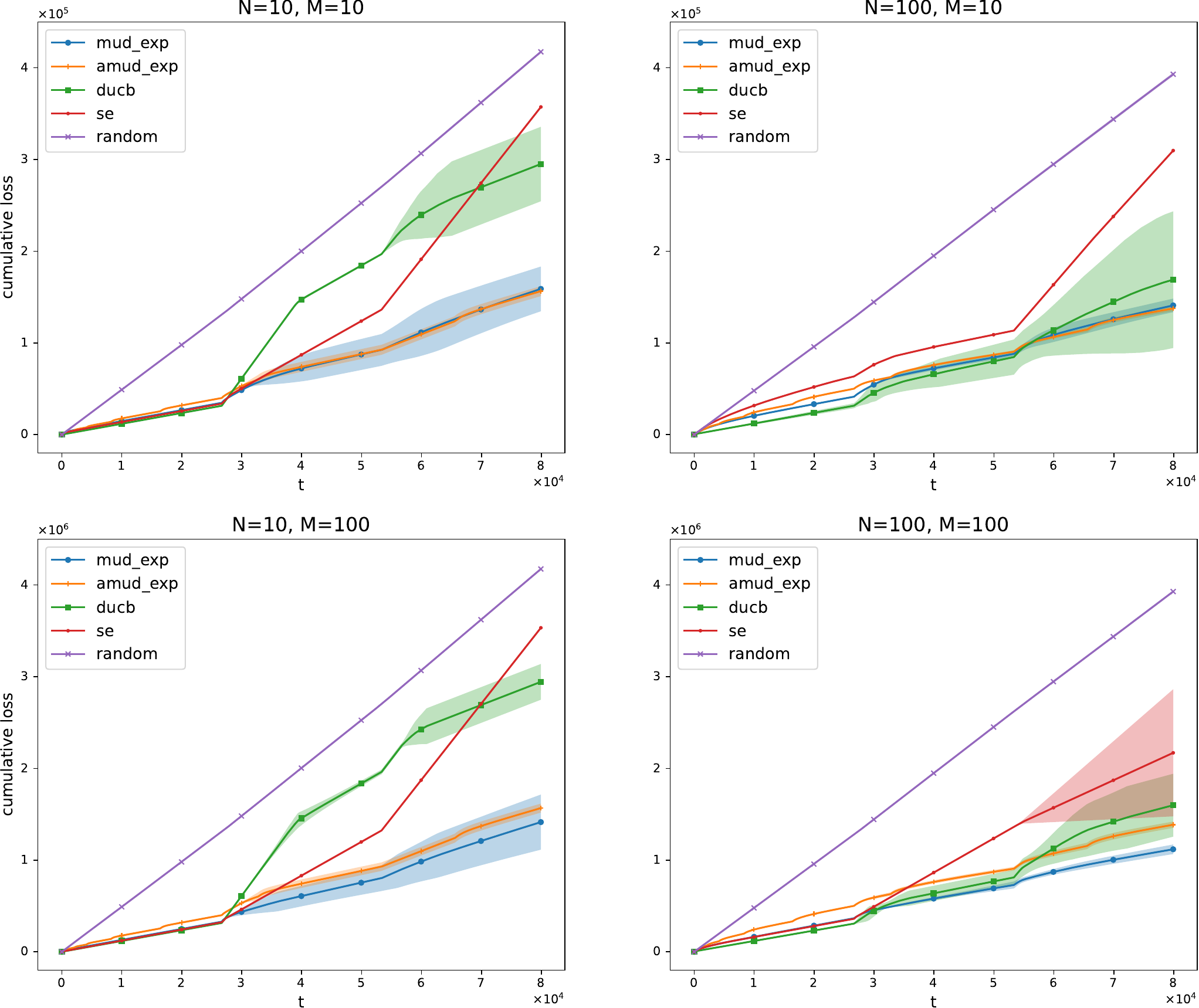}
	\caption{\centering{The cumulative loss with varying $N$ and $M$.}}
	\label{fig_n_m}
\end{figure}

With a stochastic bandit environment, in Figure \ref{fig_regret}, all the methods indicate sublinear regrets. The UCB-based method \textit{ducb} shows obviously better regret than others while \textit{amud} fails spectacularly in this stationary environment with stepped cumulative regret resulting from the evolving learning rates. For low-variance delay (delays in $[1,10]$), \textit{mud} is comparable with \textit{se} but getting surpassed by \textit{se} as increasing the variance of delay (delays in $[1,100]$ and $[1,1000]$). Apart from the superior average regret, however, \textit{ducb} suffer increasingly large variation in regret with exponential delay upper bound. In fact, it will be seen that the stability of \textit{ducb} tends to meet a threat in both stochastic and adversarial bandit environments.

In this paper, we mainly focus on adversarial bandit environments as follows. Figure \ref{fig_loss} depicts the cumulative loss with rounds under the default adversarial environment where $N=10$, $M=10$ and $\textit{tran\_num}=3$. The proposed methods \textit{mud} and \textit{amud} have the overwhelming advantage of \textit{ducb} and \textit{se}. However, this advantage is gradually eroded as the maximum delay enlarges since the proposed methods show degradation that is more sensitive to the variance of delay. Although the two proposed methods exhibit similar expected loss, \textit{amud} has a more stable performance with less variation than \textit{mud} in the results. Between UCB-based methods, although \textit{ducb} averagely outperforms \textit{se} at the end of rounds in these adversarial bandit environments including the non-default situations of the remaining figures, \textit{se} steadily possesses lower loss than \textit{ducb} right between the second and the third transition points of the distributions (note that the first transition point is the beginning as $t=0$) as well as all-time minor variation. The time period between the first and the second transition points is essentially a stochastic bandit environment, in which the performance of each method can refer to Figure \ref{fig_regret}. From the second transition point on (concretely from round $26667$ in the case of \ref{fig_loss}), the environment begins to evolve so that the proposed adversarial bandit algorithms \textit{mud} and \textit{amud} accommodate the policy to the evolution more successfully than others.

In order to investigate the influence of evolving distributions of loss on performance, we increase \textit{tran\_num} from $2$ to $10$, to compare the total loss among the three more powerful algorithms (\textit{mud}, \textit{amud} and \textit{ducb}) in Figure \ref{fig_bar1}-\ref{fig_bar3}. From Figure \ref{fig_bar1}, we can see both the proposed methods show less fluctuation in loss than \textit{ducb} when the environment is evolving, and both outperform \textit{ducb} almost all the time. However, the performance gap is reduced for increasing $d_{max}$ in Figure \ref{fig_bar2} and even gets reversed when $d_{max}$ increases to $1000$ in \ref{fig_bar3}. This is intuitive since a higher delay leads to a larger amount of less-informed rounds. Apparently, the delay of $1000$ rounds is too extreme in the practical scenarios, to be removed for further discussion. 

Moreover, we extend the evolving factor \textit{tran\_num} to $50$ and $100$, and depict the detailed tendency of loss in Figure \ref{fig_num_d10} and Figure \ref{fig_num_d100}. In both figures, there is a superiority for the proposed methods over \textit{se} all the time, but \textit{ducb} approaches when the environment evolves more frequently, even slightly surpasses the proposed ones in the Figure \ref{fig_num_d100} when \textit{tran\_num} is large. Despite the terrible stability of \textit{ducb} with the largest shade region among all the methods, it possesses impressive adaptability in contrast to \textit{se} in Figure \ref{fig_num_d10} and Figure \ref{fig_num_d100} when \textit{ducb} constantly overwhelms \textit{se} on average as the environment starts to evolve. It is interesting that in some circumstances such as $\textit{tran\_num}=50$, \textit{se} even contributes higher loss than \textit{random} method, which suggests very limited adaptability of \textit{se}. The proposed methods \textit{mud} and \textit{amud} show good adaptability as a whole, while \textit{amud} is considered to be a more stable one with less variation than \textit{mud}. Comparing Figure \ref{fig_num_d10} and Figure \ref{fig_num_d100}, it is illustrated that the larger \textit{tran\_num} (i.e. more frequently changed environment) and $d_{max}$ (i.e. less informed decision), the worse those algorithms perform, even approximating the random selection policy.

In addition, considering the increase of arm number $N$ and user number $M$ under the delayed adversarial environment would add complexity, we vary $N$ and $M$ in the default environment with $d_{max}=10$ to briefly inspect the impacts in Figure \ref{fig_n_m}. In general, from the results, the change in $M$ brings limited influence on performance, while the increase of $N$ can obviously improve the performance of \textit{ducb} and \textit{se} to approximate the proposed methods. When $M$ increases from $10$ to $100$, the primary distinction is the average performance of \textit{mud} happens to relatively improve and surpasses \textit{amud}. However, \textit{amud} still has a clear advantage in stability. Although more users could increase the statistical accuracy, this is a fundamental enhancement enjoyed by all the methods, thus no dramatic changes happen to the relative performance. On the other hand, as the number of arms $N$ increases from $10$ to $100$, the methods remain almost unvaried except for \textit{ducb} which greatly improves during the intermediate phase of the entire procedure corresponding to the second set of feedback distributions. This is probably due to more suboptimal arms somewhat filling the upper confidence bound space. From Figure \ref{fig_n_m}, note that the proposed methods \textit{mud} and \textit{amud} display invariant ability with respect to the number of arms and the number of users. 

In summary, the proposed algorithms show distinct advantages over the baselines when there are relatively small upper bounds of delay (less than 100 rounds) or properly low evolving frequency of the environment, otherwise, the state-of-the-arm UCB methods tend to encroach on the superiority, which guides the application of the proposed algorithms in the real world.

\section{Conclusion}\label{sec7}
In conclusion, this study addresses the adversarial multi-armed bandit problem with delayed feedback, where feedback results are obtained from multiple users without any internal distribution restrictions. A modified EXP3 algorithm called MUD-EXP3 is proposed to solve this problem with the oblivious loss and oblivious delay adversary setting. MUD-EXP3 employs the importance-weighted estimator for the received feedback from different users, and selects an arm at each round stochastically according to the amount of the cumulative received loss. Under the assumptions of a known terminal round index, the number of users, the number of arms, and an upper bound on the delay, the study proves a regret bound of $\mathcal{O}(\sqrt{TM^2\ln{N}(N\mathrm{e}+4d_{max}}))$, demonstrating the algorithm's effectiveness. Furthermore, considering the fact that $T$ is normally inaccessible beforehand, we propose an adaptive algorithm named AMUD-EXP3 based on a doubling trick method to accommodate MUD-EXP3. Additionally, experiments are conducted showing the advantage of our algorithms in specific adversarial bandit environments. Overall, this research provides valuable insights for addressing complex real-life scenarios with multi-user delayed feedback.

\section*{Acknowledgment}
This work was supported in part by the National Key R\&D Program of China under Grant No. 2022YFE0201400, the National Natural Science Foundation of China (NSFC) under Grant No. 62202055, the Start-up Fund from Beijing Normal University under Grant No. 310432104, the Start-up Fund from BNU-HKBU United International College under Grant No. UICR0700018-22, and the Project of Young Innovative Talents of Guangdong Education Department under Grant No. 2022KQNCX102.

\bibliographystyle{IEEEtran}
\bibliography{references}
 
\begin{IEEEbiography}[{\includegraphics[width=1in,height=1.25in,clip,keepaspectratio]{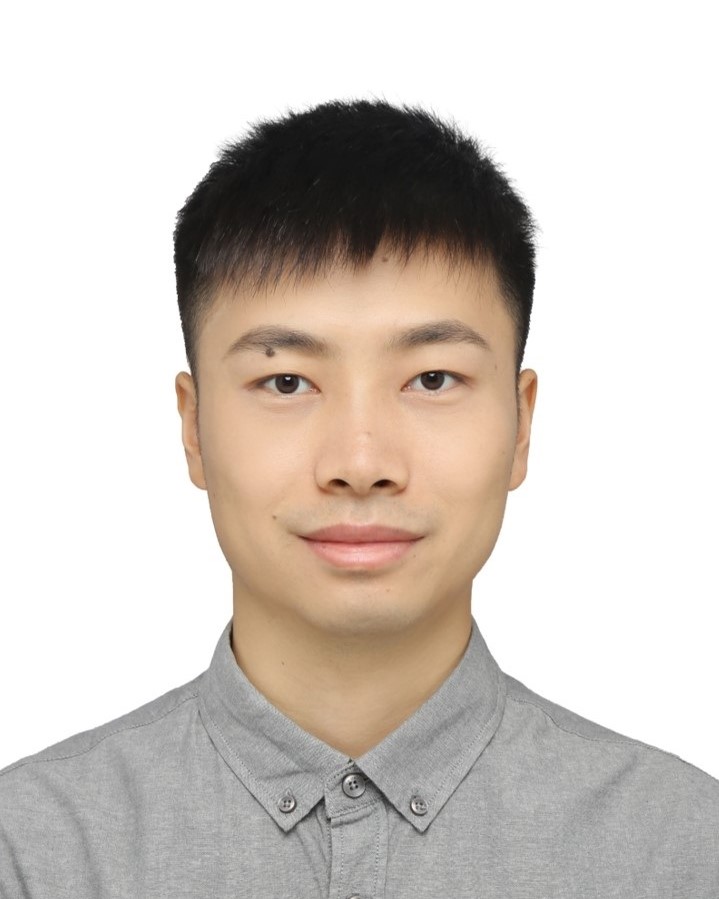}}]{Yandi Li} received his B.E. degree from the School of Optoelectronic Science and Engineering, University of Electronic Science and Technology of China, Chengdu, China, in 2013. He is currently pursuing the M.Phil. degree with the Guangdong Key Lab of AI and Multi-Modal Data Processing, Department of Computer Science, BNU-HKBU United International College, Zhuhai, China. He is supervised by Dr. Jianxiong Guo, and his research interests include social networks, online algorithms, and machine learning.
\end{IEEEbiography}

\begin{IEEEbiography}[{\includegraphics[width=1in,height=1.25in,clip,keepaspectratio]{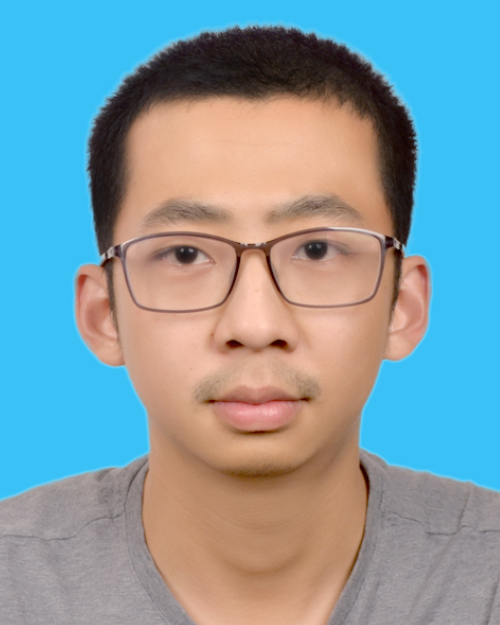}}]{Jianxiong Guo}
    (Member, IEEE) received his Ph.D. degree from the Department of Computer Science, University of Texas at Dallas, Richardson, TX, USA, in 2021, and his B.E. degree from the School of Chemistry and Chemical Engineering, South China University of Technology, Guangzhou, China, in 2015. He is currently an Assistant Professor with the Advanced Institute of Natural Sciences, Beijing Normal University, and also with the Guangdong Key Lab of AI and Multi-Modal Data Processing, BNU-HKBU United International College, Zhuhai, China. He is a member of IEEE/ACM/CCF. He has published more than 50 papers and has been a reviewer in famous international journals/conferences. His research interests include social networks, algorithm design, data mining, IoT applications, blockchain, and combinatorial optimization.
\end{IEEEbiography}

\begin{IEEEbiography}[{\includegraphics[width=1in,height=1.25in,clip,keepaspectratio]{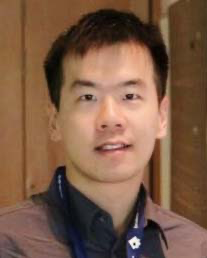}}]{Yupeng Li} (Member, IEEE) received the Ph.D. degree in computer science from The University of Hong Kong. He was with the University of Toronto and is currently with Hong Kong Baptist University. His research interests are in general areas of network science and, in particular, algorithmic decision making and machine learning problems, which arise in networked systems. %, such as information networks and ride-sharing platforms. 
He is also excited about interdisciplinary research that applies algorithmic techniques to edging problems. Recently, he has worked on robust online machine learning for the application of data classification, and he has extended these techniques to modern areas in networking and social media. Dr. Li has been awarded the Rising Star in Social Computing Award by CAAI and the distinction of Distinguished Member of the IEEE INFOCOM Technical Program Committee in 2022. He serves on the technical committees of some top conferences in computer science. His works have been published in prestigious venues, such as \textsc{IEEE INFOCOM}, \textsc{ACM MobiHoc}, \textsc{IEEE Journal on Selected Areas in Communications}, and \textsc{IEEE/ACM Transactions on Networking}. He is a member of ACM and IEEE.
\end{IEEEbiography}

\begin{IEEEbiography}[{\includegraphics[width=1in,height=1.25in,clip,keepaspectratio]{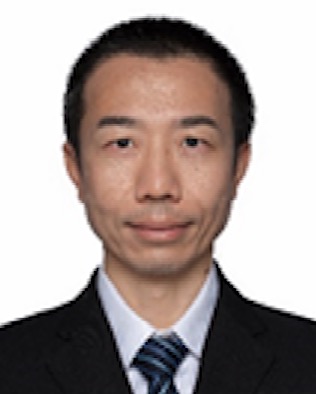}}]{Tian Wang}
	received his BSc and MSc degrees in Computer Science from the Central South University in 2004 and 2007, respectively. He received his PhD degree in City University of Hong Kong in Computer Science in 2011. Currently, he is a professor in the Institute of Artificial Intelligence and Future Networks, Beijing Normal University \& UIC. His research interests include internet of things, edge computing and mobile computing. He has 27 patents and has published more than 200 papers in high-level journals and conferences. He has more than 11000 citations, according to Google Scholar. His H-index is 53. He has managed 6 national natural science projects (including 2 sub-projects) and 4 provincial-level projects.
\end{IEEEbiography}

\begin{IEEEbiography}[{\includegraphics[width=1in,height=1.25in,clip,keepaspectratio]{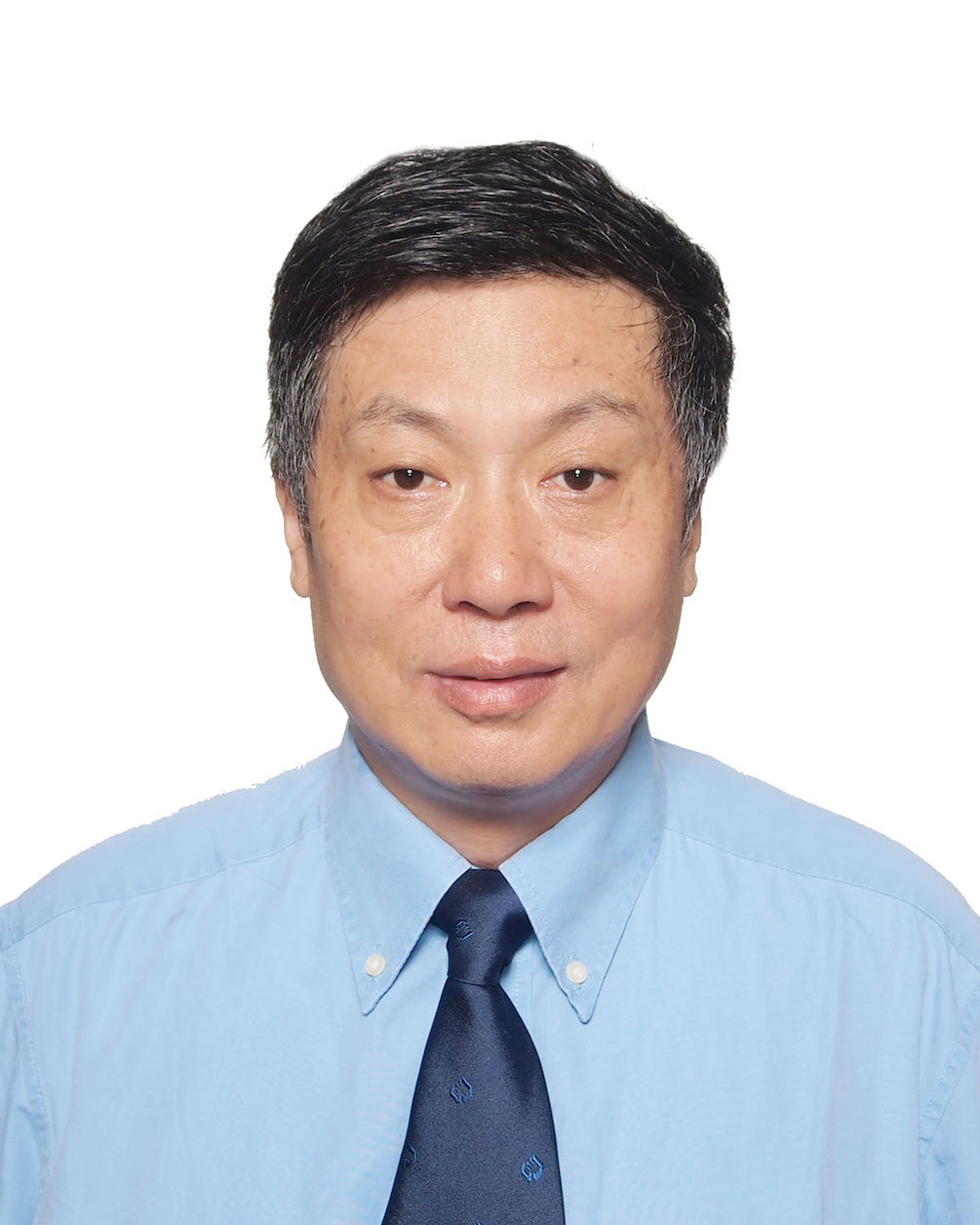}}]{Weijia Jia} (Fellow, IEEE)
    is currently a Chair Professor, Director of BNU-UIC Institute of Artificial Intelligence and Future Networks, Beijing Normal University (Zhuhai) and VP for Research of BNU-HKBU United International College (UIC) and has been the Zhiyuan Chair Professor of Shanghai Jiao Tong University, China. He was the Chair Professor and the Deputy Director of State Kay Laboratory of Internet of Things for Smart City at the University of Macau. He received BSc/MSc from Center South University, China in 82/84 and Master of Applied Sci./PhD from Polytechnic Faculty of Mons, Belgium in 92/93, respectively, all in computer science. From 93-95, he joined German National Research Center for Information Science (GMD) in Bonn (St. Augustine) as a research fellow. From 95-13, he worked in City University of Hong Kong as a professor. His contributions have been recognized as optimal network routing and deployment; anycast and QoS routing, sensors networking, AI (knowledge relation extractions; NLP, etc.), and edge computing. He has over 600 publications in the prestige international journals/conferences and research books and book chapters. He has received the best product awards from International Science \& Tech. Expo (Shenzhen) in 20112012 and the 1st Prize of Scientific Research Awards from the Ministry of Education of China in 2017 (list 2). He has served as area editor for various prestige international journals, chair, and PC member/keynote speaker for many top international conferences. He is the Fellow of IEEE and the Distinguished Member of CCF.
\end{IEEEbiography}

\vfill

\onecolumn
\appendix
\setcounter{thm}{0}
\setcounter{lem}{3}
\setcounter{cor}{0}
\renewcommand{\thesection}{\Alph{section}.\arabic{section}}
\setcounter{section}{0}

\subsection{Proof of Corollary \ref{cor1}} \label{cor1_proof}

\begin{cor}
    From Lemma \ref{lem1}, we can derive the following bound: $\sum_{i=1}^{N}\lvert p_{i}(t+1)-p_{i}(t)\rvert \leq 2\eta'\sum_{k=1}^{N}p_{k}(t)\cdot\ell_{k}(t)$.
\end{cor}
\begin{proof}
Based on Lemma \ref{lem1}, we have
    \begin{align}
        \lvert p_{i}(t+1)-p_{i}(t)\rvert &\leq \max\left\{\eta'p_{i}(t)\cdot\ell_{i}(t),\; \eta'p_{i}(t+1)\sum_{k=1}^{N}p_{k}(t)\cdot\ell_{k}(t)\right\} \notag\\
        &\leq \eta'\left(p_{i}(t)\cdot\ell_{i}(t) + p_{i}(t+1)\sum_{k=1}^{N}p_{k}(t)\cdot\ell_{k}(t)\right).\notag
    \end{align}
    Then, the sum over all arms can be bounded as
    \begin{align}
        \sum_{i=1}^{N} \lvert p_{i}(t+1)-p_{i}(t)\rvert &\leq \eta'\sum_{i=1}^{N}\left(p_{i}(t)\cdot\ell_{i}(t) + p_{i}(t+1)\sum_{k=1}^{N}p_{k}(t)\cdot\ell_{k}(t)\right) \notag\\
        &= \eta'\left(\sum_{i=1}^{N}p_{i}(t)\cdot\ell_{i}(t) + \sum_{i=1}^{N}p_{i}(t+1)\sum_{k=1}^{N}p_{k}(t)\cdot\ell_{k}(t)\right) \notag\\
        &= \eta'\left(\sum_{i=1}^{N}p_{i}(t)\cdot\ell_{i}(t) + \sum_{k=1}^{N}p_{k}(t)\ell_{k}(t)\cdot\sum_{i=1}^{N}p_{i}(t+1)\right) \notag\\
        &\mathop{=}\limits_{(a)} \eta'\left(\sum_{i=1}^{N}p_{i}(t)\cdot\ell_{i}(t) + \sum_{k=1}^{N}p_{k}(t)\cdot\ell_{k}(t)\right)\notag\\
        &= 2\eta'\sum_{k=1}^{N}p_{k}(t)\cdot\ell_{k}(t), \notag
    \end{align}
    where Eq. $(a)$ results from $\sum_{i=1}^{N}p_{i}(t+1)=1$. %$\hfill\blacksquare$
\end{proof}

\subsection{Proof of Lemma \ref{lem4}} \label{lem4_proof}
\begin{lem}
    MUD-EXP3 satisfies the following inequality
    \begin{align}
        &\mathbb{E}\left[\sum_{t=1}^{T}\sum_{(s,j)\in \Phi_t}\sum_{k=1}^{N}p_{k}(s)\cdot l_{k}^{j}(s) - \sum_{t=1}^{T}\sum_{(s,j)\in \Phi_t}\sum_{k=1}^{N}p_{k}(t)\cdot l_{k}^{j}(s)\right] \leq 2\eta'M\sum_{t=1}^{T}\sum_{(s,j)\in \Phi_t}d_s^j. \notag
    \end{align}
\end{lem}
\begin{proof}
    To prove this inequality, we first transform the expression into a form containing $\sum_{i=1}^{N}\lvert p_{i}(t+1)-p_{i}(t)\rvert$, then adopt Corollary \ref{cor1} to upper bound it.
    \begin{align}
        &\quad \mathbb{E}\left[\sum_{t=1}^{T}\sum_{(s,j)\in \Phi_t}\sum_{k=1}^{N}p_{k}(s)\cdot l_{k}^{j}(s) - \sum_{t=1}^{T}\sum_{(s,j)\in \Phi_t}\sum_{k=1}^{N}p_{k}(t)\cdot l_{k}^{j}(s)\right] \notag\\
        &= \mathbb{E}\left[\sum_{t=1}^{T}\sum_{(s,j)\in \Phi_t}\sum_{k=1}^{N}l_{k}^{j}(s)\left(p_{k}(s) - p_{k}(t)\right)\right] \notag\\
        &\leq \mathbb{E}\left[\sum_{t=1}^{T}\sum_{(s,j)\in \Phi_t}\sum_{k=1}^{N}\left(p_{k}(s) - p_{k}(t)\right)\right] \notag\\
        &= \mathbb{E}\left[\sum_{t=1}^{T}\sum_{(s,j)\in \Phi_t}\sum_{k=1}^{N}\sum_{r=s}^{t-1}\left(p_{k}(r) - p_{k}(r+1)\right)\right] \notag\\
        &\leq \mathbb{E}\left[\sum_{t=1}^{T}\sum_{(s,j)\in \Phi_t}\sum_{r=s}^{t-1}\sum_{k=1}^{N}\lvert p_{k}(r) - p_{k}(r+1)\rvert\right] \notag\\
        &\mathop{\leq}\limits_{(a)} \mathbb{E}\left[\sum_{t=1}^{T}\sum_{(s,j)\in \Phi_t}\sum_{r=s}^{t-1} 2\eta'\sum_{k=1}^{N}p_{k}(r)\cdot\ell_{k}(r)\right] \notag\\
        &\mathop{=}\limits_{(b)} 2\eta'\mathbb{E}\left[\sum_{t=1}^{T}\sum_{(s,j)\in \Phi_t}\sum_{r=s}^{t-1}\sum_{k=1}^{N}p_{k}(r)\cdot\mathbb{E}\left[\ell_{k}(r)\vert\mathcal{F}_{r-1}\right]\right] \notag\\
        &= 2\eta'\mathbb{E}\left[\sum_{t=1}^{T}\sum_{(s,j)\in \Phi_t}\sum_{r=s}^{t-1}\sum_{k=1}^{N}p_{k}(r)\cdot\mathbb{E}\left[\sum_{(s',j')\in \Phi_r}\hat{l}_{k}^{j'}(s')\vert\mathcal{F}_{r-1}\right]\right] \notag\\
        &\mathop{=}\limits_{(c)} 2\eta'\mathbb{E}\left[\sum_{t=1}^{T}\sum_{(s,j)\in \Phi_t}\sum_{r=s}^{t-1}\sum_{k=1}^{N}p_{k}(r)\sum_{(s',j')\in \Phi_r}l_{k}^{j'}(s')\right] \notag\\
        &\leq 2\eta'M\sum_{t=1}^{T}\sum_{(s,j)\in \Phi_t}d_s^j, \notag
    \end{align}
    where Ineq. $(a)$ holds by using Corollary \ref{cor1}, Eq. $(b)$ uses $p_k(r)\in \mathcal{F}_{r-1}$, and Ineq. $(c)$ follows the fact that $\hat{l}_{k}^{j'}(s')$ is $l_{k}^{j'}(s')/p_k(s')$ with probability $p_k(s')$ and zero otherwise. %$\hfill\blacksquare$
\end{proof}

\subsection{Proof of Theorem \ref{thm1}} \label{thm1_proof}
\begin{thm}
    For any arm $i\in\mathcal{N}$, MUD-EXP3 guarantees the upper bound for $\mathcal{R}_i$ is shown as 
    \begin{equation}
        \mathcal{R}_i \leq \frac{\ln{N}}{\eta^{\prime}} + \frac{1}{2}\eta^{\prime}M^2TN\mathrm{e} + 2\eta'M\sum_{t=1}^{T}\sum_{(s,j)\in \Phi_t}d_s^j + \lvert\Omega\rvert, \notag
    \end{equation}
    which implies the same regret upper bound as follows:
    \begin{equation}
        \mathcal{R} \leq \frac{\ln{N}}{\eta^{\prime}} + \frac{1}{2}\eta^{\prime}M^2TN\mathrm{e} + 2\eta'M\sum_{t=1}^{T}\sum_{(s,j)\in \Phi_t}d_s^j + \lvert\Omega\rvert. \notag
    \end{equation}
    Specially, for the known $T$ and $\sum_{t=1}^{T}\sum_{(s,j)\in \Phi_t}d_s^j$, if 
    \begin{equation*}
        \eta= \sqrt{\frac{\ln{N}}{M(TMN\mathrm{e}+4\sum_{t=1}^{T}\sum_{(s,j)\in \Phi_t}d_s^j)}} \leq \frac{1}{MN\mathrm{e}(\sum_{t=1}^{T}\sum_{(s,j)\in \Phi_t}d_s^j+1)},
    \end{equation*}
    we have
    \begin{equation}
        \mathcal{R} \leq \mathcal{O}\left(\sqrt{M\ln{N}(TMN\mathrm{e}+4\sum_{t=1}^{T}\sum_{(s,j)\in \Phi_t}d_s^j})\right) \leq \mathcal{O}\left(\sqrt{TM^2\ln{N}\left(N\mathrm{e}+4d_{max}\right)}\right). \notag
    \end{equation}
\end{thm}
\begin{proof}
    The expression of $\mathcal{R}_i$ can be transformed in order to approach Lemma \ref{lem3} and Lemma \ref{lem4} for upper bounding, as shown below.
    \begin{align}
        \mathcal{R}_i &= \mathbb{E}\left[\sum_{t=1}^{T}\sum_{j=1}^{M}l_{A_t}^{j}(t)\right] - \sum_{t=1}^{T}\sum_{j=1}^{M}l_{i}^{j}(t) \notag\\
        &= \mathbb{E}\left[\sum_{t=1}^{T}\sum_{j=1}^{M}\mathbb{E}\left[l_{A_t}^{j}(t)\vert\mathcal{F}_{t}\right]\right] - \sum_{t=1}^{T}\sum_{j=1}^{M}l_{i}^{j}(t) \notag\\
        &= \mathbb{E}\left[\sum_{t=1}^{T}\sum_{j=1}^{M}\sum_{k=1}^{N}p_{k}(t)\cdot l_{k}^{j}(t) - \sum_{t=1}^{T}\sum_{j=1}^{M}l_{i}^{j}(t)\right] \notag\\
        &= \mathbb{E}\left[\sum_{t=1}^{T}\sum_{(s,j)\in \Phi_t}\sum_{k=1}^{N}p_{k}(s)\cdot l_{k}^{j}(s) + \sum_{(t,j)\in \Omega}\sum_{k=1}^{N}p_{k}(t)\cdot l_{k}^{j}(t) - \sum_{t=1}^{T}\sum_{j=1}^{M}l_{i}^{j}(t)\right] \notag\\
        &\mathop{\leq}\limits_{(a)} \mathbb{E}\left[\sum_{t=1}^{T}\sum_{(s,j)\in \Phi_t}\sum_{k=1}^{N}p_{k}(s)\cdot l_{k}^{j}(s) - \sum_{t=1}^{T}\sum_{j=1}^{M}l_{i}^{j}(t)\right] + \lvert\Omega\rvert \notag\\
        &= \mathbb{E}\left[\sum_{t=1}^{T}\sum_{(s,j)\in \Phi_t}\sum_{k=1}^{N}p_{k}(t)\cdot l_{k}^{j}(s) - \sum_{t=1}^{T}\sum_{(s,j)\in \Phi_t}\sum_{k=1}^{N}p_{k}(t)\cdot l_{k}^{j}(s)\right] \notag\\
        &\qquad + \mathbb{E}\left[\sum_{t=1}^{T}\sum_{(s,j)\in \Phi_t}\sum_{k=1}^{N}p_{k}(s)\cdot l_{k}^{j}(s) - \sum_{t=1}^{T}\sum_{j=1}^{M}l_{i}^{j}(t)\right] + \lvert\Omega\rvert \notag\\
        &= \mathbb{E}\left[\sum_{t=1}^{T}\sum_{(s,j)\in \Phi_t}\sum_{k=1}^{N}p_{k}(t)\cdot l_{k}^{j}(s) - \sum_{t=1}^{T}\sum_{j=1}^{M}l_{i}^{j}(t)\right] \notag\\
        &\qquad + \mathbb{E}\left[\sum_{t=1}^{T}\sum_{(s,j)\in \Phi_t}\sum_{k=1}^{N}p_{k}(s)\cdot l_{k}^{j}(s) - \sum_{t=1}^{T}\sum_{(s,j)\in \Phi_t}\sum_{k=1}^{N}p_{k}(t)\cdot l_{k}^{j}(s)\right] + \lvert\Omega\rvert \notag\\
        &\mathop{\leq}\limits_{(b)} \frac{\ln{N}}{\eta^{\prime}} + \frac{1}{2}\eta^{\prime}M^2TN\mathrm{e} + 2\eta'M\sum_{t=1}^{T}\sum_{(s,j)\in \Phi_t}d_s^j + \lvert\Omega\rvert, \notag
    \end{align}
    where Ineq. $(a)$ follows since $l_{i}^{j}(t)\leq 1$ and Ineq. $(b)$ results from Lemma \ref{lem3} and Lemma \ref{lem4}. %$\hfill\blacksquare$
\end{proof}

\subsection{Proof of Theorem \ref{thm2}} \label{thm2_proof}

\begin{thm}
    For any arm $i\in\mathcal{N}$, AMUD-EXP3 guarantees the upper bound for $\mathcal{R}_i$ as shown below.
    \begin{equation}
        \mathcal{R}_i \leq \left(11\sqrt{M\ln{N}}+7\sqrt{M}\right)\sqrt{\sum_{t=1}^{T}\sum_{j=1}^{M}d_t^j} + \frac{5}{2}MN\mathrm{e}\sqrt{T\ln{N}}. \notag
    \end{equation}
\end{thm}
\begin{proof}
    Based on Ineq. (\ref{epc-rgt}), the total regret $\mathcal{R}_i$ can be represented as the sum of each epoch regret $\mathcal{R}_{\varepsilon}$. Letting $\eta_{\varepsilon}=\frac{1}{M}\sqrt{\frac{\ln{N}}{2^{\varepsilon}}}$, then we have
    \begin{align}
        \mathcal{R}_i &= \sum_{\varepsilon=1}^{E}\mathcal{R}_{\varepsilon}
        \leq \sum_{\varepsilon=1}^{E}\left[\frac{\ln{N}}{\eta_{\varepsilon}} + \frac{1}{2}\eta_{\varepsilon}M^2\lvert\mathcal{T}_{\varepsilon}\rvert N\mathrm{e} + 2\eta_{\varepsilon}M\sum_{t\in\mathcal{T}_{\varepsilon}}\sum_{(s,j)\in \Phi_t}d_s^j + \lvert\Omega_{\varepsilon}\rvert \right] \notag\\
        &\mathop{\leq}\limits_{(a)} \sum_{\varepsilon=1}^{E}\left[M\sqrt{\ln{N}}\cdot 2^{\frac{\varepsilon}{2}}+\frac{1}{2}MN\mathrm{e}\sqrt{\ln{N}}\lvert\mathcal{T}_{\varepsilon}\rvert 2^{-\frac{\varepsilon}{2}}+2\cdot 2^{-\frac{\varepsilon}{2}}M\sqrt{\ln{N}}\left(2^{\varepsilon-1}+\frac{1}{\varepsilon}\right)+2\cdot 2^{\frac{\varepsilon}{2}}M \right] \notag\\
        &= \sum_{\varepsilon=1}^{E}\left[M\sqrt{\ln{N}}\left(2^{\frac{\varepsilon}{2}}+2\cdot 2^{-\frac{\varepsilon}{2}}\left(2^{\varepsilon-1}+\frac{1}{\varepsilon}\right)\right)+\frac{1}{2}MN\mathrm{e}\sqrt{\ln{N}}\lvert\mathcal{T}_{\varepsilon}\rvert 2^{-\frac{\varepsilon}{2}}+2\cdot 2^{\frac{\varepsilon}{2}}M \right] \notag\\
        &\leq \sum_{\varepsilon=1}^{E}\left[3M\sqrt{\ln{N}}\cdot 2^{\frac{\varepsilon}{2}}+\frac{1}{2}MN\mathrm{e}\sqrt{\ln{N}}\lvert\mathcal{T}_{\varepsilon}\rvert 2^{-\frac{\varepsilon}{2}}+2\cdot 2^{\frac{\varepsilon}{2}}M \right] \notag\\
        &= \left(3M\sqrt{\ln{N}}+2M\right)\sum_{\varepsilon=1}^{E}2^{\frac{\varepsilon}{2}} + \frac{1}{2}MN\mathrm{e}\sqrt{\ln{N}}\sum_{\varepsilon=1}^{E}\lvert\mathcal{T}_{\varepsilon}\rvert 2^{-\frac{\varepsilon}{2}} \notag\\
        &= \left(3M\sqrt{\ln{N}}+2M\right)\frac{2^{\frac{E}{2}}-1}{\sqrt{2}-1} + \frac{1}{2}MN\mathrm{e}\sqrt{\ln{N}}\sum_{\varepsilon=1}^{E}\lvert\mathcal{T}_{\varepsilon}\rvert 2^{-\frac{\varepsilon}{2}} \notag\\
        &= \left(3M\sqrt{\ln{N}}+2M\right)\left(\sqrt{2}+1\right)\left(\sqrt{2}\cdot \sqrt{2^{E-1}}-1\right) + \frac{1}{2}MN\mathrm{e}\sqrt{\ln{N}}\sum_{\varepsilon=1}^{E}\lvert\mathcal{T}_{\varepsilon}\rvert 2^{-\frac{\varepsilon}{2}} \notag\\
        &\mathop{\leq}\limits_{(b)} \left(3M\sqrt{\ln{N}}+2M\right)\left(\sqrt{2}+1\right)\left(\sqrt{2}\cdot\sqrt{\frac{1}{M}\sum_{t=1}^{T}\sum_{j=1}^{M}d_t^j}-1\right) + \frac{1}{2}MN\mathrm{e}\sqrt{\ln{N}}\cdot 5\sqrt{T} \notag\\
        &\leq \left(11\sqrt{M\ln{N}}+7\sqrt{M}\right)\sqrt{\sum_{t=1}^{T}\sum_{j=1}^{M}d_t^j} + \frac{5}{2}MN\mathrm{e}\sqrt{T\ln{N}}, \notag
    \end{align}
    where Ineq. $(a)$ holds due to $\sum_{t\in\mathcal{T}_{\varepsilon}}\sum_{(s,j)\in \Phi_t}d_s^j\leq\left(2^{\varepsilon-1}+\frac{1}{\varepsilon}\right)M$ in Lemma \ref{lem5} and $\lvert\Omega_{\varepsilon}\rvert \leq 2^{\frac{\varepsilon}{2}}\cdot 2M$ in Lemma \ref{lem6}, while Ineq. $(b)$ holds due to Lemma \ref{lem7}. %$\hfill\blacksquare$
\end{proof}

\end{document}